\documentclass{article}

\usepackage{microtype}
\usepackage{graphicx}
\usepackage{subfigure}
\usepackage{booktabs} 

\usepackage{hyperref}

\usepackage[accepted]{icml2024}

\usepackage{amsmath}
\usepackage{amssymb}
\usepackage{mathtools}
\usepackage{amsthm}
\usepackage{thm-restate}
\usepackage{bm}
\usepackage[capitalize,noabbrev]{cleveref}

\makeatletter
\def\thmheadbrackets#1#2#3{%
  \thmname{#1}\thmnumber{\@ifnotempty{#1}{ }\@upn{#2}}%
  \thmnote{ {\the\thm@notefont[#3]}}}
\makeatother

\newtheoremstyle{brakets}
  {}
  {}
  {\itshape}
  {}
  {\bfseries}
  {.}
  { }
  {\thmheadbrackets{#1}{#2}{#3}}

\theoremstyle{brakets}

\newtheorem{lemma}{Lemma}
\newtheorem{theorem}{Theorem}
\newtheorem{proposition}[theorem]{Proposition}
\newtheorem{claim}{Claim}

\newtheorem{definition}{Definition}

\renewcommand{\bar}{\overline}
\renewcommand{\hat}{\widehat}
\renewcommand{\tilde}{\widetilde}

\DeclareMathOperator*{\argmax}{arg\,max}

\newcommand{\statesize}{|\Omega|}
\newcommand{\actionsize}{|\mathcal{A}|}

\newcommand{\bs}{\bm{s}}
\newcommand{\bpi}{\bm{\pi}}

\newcommand{\ubar}{\bar{u}}

\newcommand{\eps}{\varepsilon}

\newcommand{\Sig}{\mathcal{S}}

\newcommand{\A}{\mathcal{A}}

\DeclareMathOperator{\E}{\mathbb{E}}

\newcommand{\shortp}{\textup{\texttt{+}}}
\newcommand{\name}{\textsc{DNL}}

\usepackage{amsmath,amsfonts,bm}

\def\eqref#1{equation~\ref{#1}}

\def\1{\bm{1}}

\def\eps{{\epsilon}}

\def\vb{{\bm{b}}}

\def\vh{{\bm{h}}}

\def\vo{{\bm{o}}}

\def\vr{{\bm{r}}}

\def\vx{{\bm{x}}}
\def\vy{{\bm{y}}}
\def\vz{{\bm{z}}}

\def\mM{{\bm{M}}}

\def\mR{{\bm{R}}}

\def\mW{{\bm{W}}}

\DeclareMathAlphabet{\mathsfit}{\encodingdefault}{\sfdefault}{m}{sl}
\SetMathAlphabet{\mathsfit}{bold}{\encodingdefault}{\sfdefault}{bx}{n}

\setcitestyle{number}
\newcount\Comments  
\newcommand{\kibitz}[2]{\ifnum\Comments=1\textcolor{#1}{#2}\fi}

\newcount\CommentsAdd  
\newcommand{\kibitzAdd}[2]{\ifnum\CommentsAdd=1{\color{#1}{#2}}\fi}
\definecolor{english}{rgb}{0.0, 0.5, 0.0}

\newcommand{\squishlist}{
   \begin{list}{$\bullet$}
    { \setlength{\itemsep}{0pt}      \setlength{\parsep}{3pt}
      \setlength{\topsep}{3pt}       \setlength{\partopsep}{0pt}
      \setlength{\leftmargin}{1.5em} \setlength{\labelwidth}{1em}
      \setlength{\labelsep}{0.5em} } }
\newcommand{\squishend}{  \end{list}  }

\usepackage[textsize=tiny]{todonotes}
\allowdisplaybreaks

\icmltitlerunning{Multi-Sender Persuasion}

\begin{document}

\twocolumn[
\icmltitle{
Multi-Sender Persuasion: A Computational Perspective
}



\icmlsetsymbol{equal}{*}

\begin{icmlauthorlist}
\icmlauthor{Safwan Hossain}{equal,yyy}
\icmlauthor{Tonghan Wang}{equal,yyy}
\icmlauthor{Tao Lin}{equal,yyy}
\icmlauthor{Yiling Chen}{yyy}
\icmlauthor{David C. Parkes}{yyy}
\icmlauthor{Haifeng Xu}{zzz}
\end{icmlauthorlist}

\icmlaffiliation{yyy}{Harvard University}
\icmlaffiliation{zzz}{University of Chicago}

\icmlcorrespondingauthor{Safwan Hossain, Tonghan Wang, Tao Lin}{\{shossain, twang1, tlin\}@g.harvard.edu}

\icmlkeywords{Machine Learning, ICML}

\vskip 0.3in
]



\printAffiliationsAndNotice{\icmlEqualContribution} 

\begin{abstract}
We consider the \emph{multi-sender persuasion} problem: multiple players with informational advantage signal to convince a single self-interested actor to take certain actions.
This problem generalizes the seminal \emph{Bayesian Persuasion} framework and is ubiquitous in computational economics, multi-agent learning, and multi-objective machine learning. The core solution concept here is the Nash equilibrium of senders' signaling policies. Theoretically, we prove that finding an equilibrium in general is PPAD-Hard; in fact, even computing a sender's best response is NP-Hard. Given these intrinsic difficulties, we turn to finding local Nash equilibria.  We propose a novel differentiable neural network to approximate this game's non-linear and discontinuous utilities. Complementing this with the extra-gradient algorithm, we discover local equilibria that Pareto dominates full-revelation equilibria and those found by existing neural networks. Broadly, our theoretical and empirical contributions are of interest to a large class of economic problems.

\end{abstract}

\section{Introduction}
Bayesian Persuasion (BP)~\cite{kamenica2011bayesian} has emerged as a seminal concept in economics and decision theory. At its heart, it is a principal-agent problem that models an informed sender strategically revealing some information to affect the decisions of a self-interested receiver. Both parties are assumed to be Bayesian and have distinct utilities that depend on some realized \emph{state of nature}, and the action taken by the receiver. The sender privately observes the state and can commit to selectively disclosing this information through a randomized \emph{signaling policy}. The receiver updates their posterior belief based on the realized signal and best responds with an optimal action for this belief. The sender's goal is to maximize their utility by designing a signaling policy that nudges the receiver toward decisions preferred by the sender. This information design problem has found widespread applicability in a myriad of domains including recommendation systems~\cite{mansour2015bayesian, mansour2016bayesian}, auctions and advertising \citep{wang2024gemnet,bro2012send, emek2014signaling, badanidiyuru2018targeting}, social networks \citep{candogan2020optimal, acemoglu2021model}, and reinforcement learning \citep{castiglioni2020online, wu2022sequential}. 

The standard BP model is however significantly constrained by a strong assumption: the presence of only one sender. In the applications mentioned above and indeed more broadly in settings like multi-agent learning~\cite{balduzzi2018mechanics} and machine learning with multiple objectives~\cite{pfau2016connecting,jaderberg2017decoupled}, it is natural to have multiple parties who wish to influence the receiver toward their respective, possibly conflicting goals. As a demonstrative example, consider two ride-sharing firms, Uber and Lyft, and a dual-registered driver. While the driver is unaware of real-time demand patterns, both firms have access to and can strategically signal this to the driver and influence them toward certain pick-ups. The platforms' goals however are not aligned, with each wishing to direct the driver to their respective optimal pick-ups. The driver is also self-interested and may prefer pick-ups that are on the way home. Our work aims to study this tension induced by multiple informed parties attempting to influence a self-interested receiver's decision-making, within the BP paradigm. Crucially, while the sender-receiver relation still outlines a sequential game, the interaction \emph{among the multiple senders} in our setting forms a simultaneous game, with the resulting Nash Equilibrium (NE) being of core interest. 

While this setup has been modeled in economic literature \citet{gentzkow_bayesian_2017, ravindran_competing_2022}, the multi-sender persuasion problem has not been formally studied from a computational perspective and presents distinct challenges. In standard single-sender BP, the optimal signaling policy for the sender can be computed efficiently by a linear program \cite{dughmi2016algorithmic}, which no longer holds in the multi-sender case where we need to compute a sender's best-responding signaling policy given others' policies.  We give a non-convex optimization program for the best response problem (Proposition~\ref{prop:best-response-bi-linear}) and through an involved reduction, prove that computing best response is in-fact NP-Hard in multi-sender persuasion games (Theorem~\ref{thm:best-response-NP-hard}). For the equilibrium computation, we significantly generalize a specific characterization from prior works to show that a trivial equilibrium can be found easily under certain conditions, but it might offer poor utility to the senders (Theorem~\ref{thm:full-revelation}). We then prove that finding an equilibrium in general settings is PPAD-hard (Theorem~\ref{thm:PPAD-non-fixed}). These computational hardness results are our main theoretical contribution.

The intrinsic difficulty of finding (global) equilibrium in multi-sender persuasion motivates us to propose a deep-learning approach to finding $\epsilon$-\emph{local} equilibria (no beneficial unilateral deviation in a limited range).
This spiritually straddles two bodies of work - the emergent area of \emph{differentiable economics} that builds a parameterized representation for optimal economic design~\cite{wang2024gemnet}, and the rich literature on learning in games~\cite{bowling2004convergence,balduzzi2018mechanics,song2019convergence,azizian2020tight,fiez2020implicit,bai2021sample,bichler2021learning,haghtalab2022learning,goktas2023generative}. Mirroring the obstacles encountered in theoretical analysis, the non-differentiable and indeed discontinuous nature of the utility functions (Proposition~\ref{prop:discontinious}) also pose hurdles to identifying even local equilibria. To address this, we propose a novel end-to-end differentiable network architecture that is expressive enough to model the abrupt changes in utilities. Once trained, these networks can complement algorithms like 
extra-gradient~\cite{korpelevich1976extragradient,jelassi2020extragradient} to locate $\epsilon$-local NE.
The quality of the approximated utility landscape confirms the superior expressive capacity of our networks. Further, we demonstrate that this improvement helps to discover $\epsilon$-local NE that Pareto dominates the full-revelation equilibria (Theorem~\ref{thm:full-revelation}) and the $\epsilon$-local NE found in both synthetic and real-world scenarios by existing continuous and discontinuous~\cite{wang2023deep} networks. Our novel techniques may be of independent interest for learning in general games with discontinuous and non-linear utilities. 

\subsection{Additional Related Work}
The study of Bayesian persuasion and its various iterations has been extensively explored in the literature, as evidenced by the comprehensive surveys of \citet{dughmi_2017_survey, kamenica2019bayesian, bergemann_information_2019}. Among these, the most closely aligned with our work are the investigations involving multiple senders. The model of~\citet{gentzkow_bayesian_2017} explores a scenario where senders can arbitrarily correlate their signals, whereas~\citet{li2021sequential} consider sequential senders who choose signaling policies after observing those of previous senders. These two models differ from ours wherein the senders send signals to the receiver \emph{independently} and \emph{simultaneously} conditioning on the realized state. Further, they do not provide significant computational insights. 

\citet{ravindran_competing_2022} also study Bayesian persuasion games featuring multiple independent and simultaneous senders but assume that the senders have zero-sum utilities. They show that with a sufficiently large signaling space, the only Nash equilibrium is full revelation, wherein the state of nature is fully revealed to the receiver. However, many multi-sender persuasion games do not conform to a zero-sum utility framework. Consequently, two important structural questions arise from their work: (1) whether such a full-revelation equilibrium exists under a limited signaling space, and (2) whether multi-sender persuasion games with general utility structures give rise to other types of equilibria that extend beyond full revelation. We provides affirmative answers to these queries, as delineated in Theorem~\ref{thm:full-revelation}.

There is a literature on competitive information design where multiple senders have private states; see, e.g., \citet{au_competitive_2020, ding2023competitive, du_competitive_2024} and the references therein.  Their models are different from ours where the senders share a common state.  

Bayesian persuasion is subsumed within the broader principal-agent model \cite{gan2024generalized}, a concept that addresses a multitude of economics problems, including contract design~\cite{zhu2022sample,wang2023deep} and Stackelberg games~\cite{myerson1982optimal}. In economic theory, the notion of incorporating multiple principals, analogous to senders in our context, has been proposed to model a range of important settings \citep{waterman1998principal, hu2023principal}. However, similar to the existing work on multi-sender persuasion, these contributions typically retain a conceptual focus from an economic perspective. Our work diverges by taking a computational lens. We introduce rigorous hardness guarantees for the best-response computation and equilibrium determination and propose a novel deep learning approach for identifying $\epsilon$-local equilibria that may hold wider applicability. 

\begin{figure*}
    \centering
    \includegraphics[width=\linewidth]{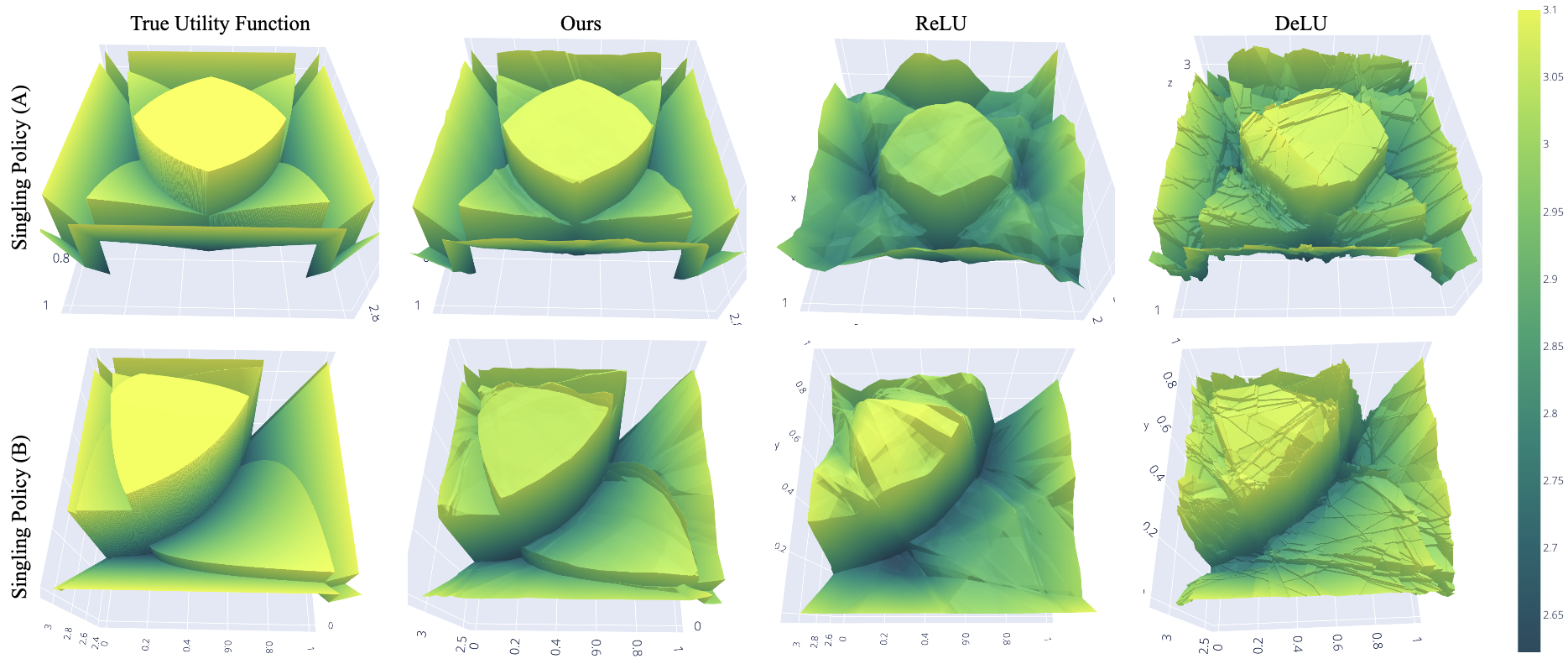}
    \vspace{-2em}
    \caption{Discontinuous utility functions in a multi-sender persuasion game with 2 senders, 2 signals, 2 actions, and 2 states. In each subplot: the x-axis represents the probability of $\mathtt{Sender 1}$ transmitting $\mathtt{Signal 1}$ at $\mathtt{State 1}$, the y-axis shows the probability of $\mathtt{Sender 2}$ emitting $\mathtt{Signal 1}$ at $\mathtt{State 1}$, and the z-axis quantifies $\mathtt{Sender 2}$'s utility. Signaling strategies of both senders at $\mathtt{State 2}$ are set to $(0.5, 0.5)$ in the top row and to $(0.2, 0.8)$ and $(0.8, 0.2)$ in the bottom row. In each column, we show the groundtruth ex-ante utility, and the approximation results achieved by our method, ReLU, and DeLU~\cite{wang2023deep} networks, respectively.}
    \vspace{-1em}
    \label{fig:viz}
\end{figure*}
\section{Model}
\paragraph{\textbf{Preliminaries}} There are $n$ senders $\{1, \dots, n\}$ and a receiver. Let $\Omega$ be a finite set of possible states, with $\omega \in \Omega$ denoting an arbitrary one. All senders and the receiver share a common prior distribution $\mu_0$ over the states $\Omega$. We use $\mu \in \Delta(\Omega)$ to denote a distribution over states.
Receiver takes some action $a \in \mathcal{A}$ whose utility depends on the realized state $\omega$, and is given by $v: \Omega \times \mathcal{A} \rightarrow \mathbb{R}$. The receiver's utility can also be represented as an $\statesize \times \actionsize$ matrix $V$, with $V[i,j]$ denoting the utility the receiver has for action $j$ at state $i$. The utility function of the $j$th sender $u_{j}: \Omega \times \mathcal{A} \rightarrow \mathbb{R}$ also depends on the realized state and the receiver's action. While the receiver only knows the prior, senders privately observe the state realization $\omega \sim \mu_0$ and can use this informational advantage to alter the receiver's belief and persuade it to take certain actions. 

\paragraph{\textbf{Persuasion}} We model the interaction between senders and the receiver using the seminal Bayesian Persuasion (BP) framework. Senders can leverage their private observation of $\omega$ by strategically signaling the receiver. Formally, letting $\Sig$ be a finite signal space, each sender $j$ has an independent \emph{signaling policy} $\pi_j(s_j|\omega)$ which specifies the probability of sending signal $s_j \in \Sig$ when the realized state is $\omega$.
From the receiver's perspective, it observes a joint signal $\bm{s} = (s_1, \dots, s_n)$ sampled from the joint conditional distribution $\bpi(\bm{s}|\omega) = \prod_{j=1}^n{\pi_j(s_j|\omega)}$.
While many works on Bayesian persuasion assume the signal space $|\Sig| \ge |\mathcal{A}|$ 
\cite{kamenica2011bayesian, dughmi2016algorithmic}, we study the multi-sender problem in full generality (allowing $|\Sig| < |\A|$), since in many settings, action space can be arbitrarily large or even continuous (common in economic literature), but signaling/communication space may be limited.
Consistent with the classical BP model, we assume that senders announce and commit to their signaling policies before observing state realizations. 
The receiver is considered Bayesian rational, and upon signal realization, it updates its belief about the state and takes a resulting optimal action according to its utility. We denote the interaction between senders and the receiver as a \emph{multi-sender persuasion game} and summarize it as follows: 
\squishlist
    \item All senders simultaneously announce their signaling policies $\bm \pi = (\pi_1, \ldots, \pi_n)$. 
    \item State $\omega \sim \mu_0$ is observed by senders but not the receiver. 
    \item Each sender $j$ simultaneously draws a signal $s_j \sim \pi_j(\cdot | \omega)$ to send to the receiver. For $\bs = (s_1, \ldots, s_n)$, $\bpi(\bs | \omega) = \prod_{j=1}^n \pi_j(s_j | \omega)$ is the joint signal probability. 
    \item After observing joint signal $\bm s$, the receiver forms posterior belief $\mu_{\bm s}$ about the state ($\mu_{\bm s}(\omega) = \frac{\mu_0(\omega) \bm \pi(\bm s| \omega)}{\bm \pi(\bm s)}$ for every $\omega \in \Omega$) and takes an optimal action 
    \begin{align}
        a^*(\mu_{\bm s}) = \argmax_{a\in \A} \E_{\omega \sim \mu_{\bs}}v(\omega, a).\nonumber
    \end{align}
    \item Each sender $j$ obtains utility $u_j(\omega, a^*(\mu_{\bm s}))$. 
\squishend

The senders attempt to use signaling to maximize their ex-ante utility, described below.

\begin{definition}\label{def:ex-ante_u}
The ex-ante utility for sender $j$ under joint signaling policy $\bpi = (\pi_j, \bpi_{-j})$ is $\ubar_j(\bpi) = \sum_{\omega \in \Omega} \mu_0(\omega) \sum_{\bm s \in \Sig^n} \bm \pi(\bm s | \omega) u_j(\omega, a^*(\mu_{\bs}))$, where $\mu_{\bs}$ is the posterior distribution induced by joint signal $\bs$ and policy $\bpi$, and $a^*(\mu_{\bs}) $ is the receiver's optimal (utility-maximizing) action at belief $\mu_{\bs}$. 
\end{definition}

The relationship between senders and the receiver forms a multi-leader-single-follower game since senders reveal their policies first and the receiver subsequently best responds to joint signal realizations generated by these policies. 
While the senders and the receiver have a sequential relationship, the senders choose their signaling policies simultaneously.
Thus, we consider Nash equilibria among the senders: 
\begin{definition}
    A Nash equilibrium (NE) for the multi-sender persuasion game is a profile of signaling policies $\bpi = (\pi_1, \ldots, \pi_n)$ such that for any sender $j$ and deviating policy $\pi_j'$, $\ubar_j(\bpi) \ge \ubar_j(\pi'_j, \bpi_{-j})$. 
\end{definition}

The equilibrium defined above is in fact a subgame perfect equilibrium of the extensive-form game among the senders and the receiver. We use the term ``Nash equilibrium'' to emphasize the simultaneity of the senders' interaction.

\section{Theoretical Results}
We now look to theoretically understand the equilibrium properties of the multi-sender persuasion game. We first consider the canonical best response problem and show that solving it, even approximately, is NP-Hard. We then extend and generalize a known equilibrium characterization that relies on revealing maximal information to the receiver. This equilibrium is generally not ideal for senders and is possible only under certain conditions. Furthermore, in the general case, 
we show that equilibrium computation is PPAD Hard. Cumulatively, our strong intractability results together suggest that developing provably efficient algorithms for finding global equilibria would be extremely challenging in our setting.

\subsection{Best Response}
We first consider the \emph{best response} problem for an sender; namely, fixing other senders' signaling schemes $\bm \pi_{-i}$, what is the optimal signaling scheme $\pi_i$ that maximizes the ex-ante utility $\ubar_i(\pi_i, \bm \pi_{-i})$ of sender $i$? The best-response problem is essential to verifying whether a given joint signaling scheme $\bm \pi = (\pi_1, \ldots, \pi_n)$ is a Nash equilibrium. Further, standard equilibrium solving techniques often rely on simulating best response dynamics.

In normal-form games, fixing others' strategies $\bm x_{-i}$, the utility $u_i(x_i, \bm x_{-i})$ of a player is linear in $x_i$, so the best response problem can be solved by a linear program efficiently. In persuasion, a sender's signaling policy changes the induced posteriors, which changes the optimal action the receiver takes since the receiver maximizes expected utility. Correspondingly, a sender's utility function $\ubar_i(\pi_i, \bm \pi_{-i})$ is piece-wise linear with discontinuities corresponding to signaling schemes wherein the mapping from signal realization to optimal receiver actions changes. This is more generally formalized in Proposition \ref{prop:discontinious} (with proof in Appendix \ref{app:discont_prop}). 
\begin{proposition}[Discontinuous Utility]\label{prop:discontinious}
    The sender's utility function $\ubar_i(\bm \pi)$ is discontinuous and piecewise non-linear in $(\pi_1, \dots, \pi_n)$. Fixing $\bm \pi_{-i}$, $\ubar_i(\pi_i, \bm \pi_{-i})$ is discontinuous and piecewise linear in $\pi_i$. 
\end{proposition}
Maximizing $u_i(\pi_i, \bm \pi_{-i})$ by enumerating all linear pieces is infeasible because, by a rough estimate, the number of linear pieces can be as large as $O\big((|\Sig|^n |\A|^2)^{|\Omega||\Sig|}\big)$. Instead of enumerating all $O\big((|\Sig|^n |\A|^2)^{|\Omega||\Sig|}\big)$ linear pieces, we design a continuous bi-linear program to solve the best response problem (with proof in Appendix \ref{app:best_resp_bi_linear}).
\begin{restatable}[Best Response Program]{proposition}{bestresponsebilinear}\label{prop:best-response-bi-linear}
Let $\Delta v(\omega, a, a')$ $\triangleq v(\omega, a) - v(\omega, a')$ for actions $a, a'$. Then given others' signaling schemes $\bm \pi_{-i}$, sender $i$'s best response can be solved by the following optimization program with $|\Omega||\Sig| + |\Sig|^n|\A|$ continous variables and $O(|\Sig|^n|\A|)$ constraints: 
\begin{align*}
    & \hspace{-0.5em} \max_{\pi_i, y} ~  \sum_{\omega \in \Omega} \sum_{\bm s \in \Sig^n} \mu_0(\omega) \bm \pi_{-i}(\bm s_{-i}|\omega) \pi_i(s_i|\omega) \sum_{a\in \A} u_i(\omega, a) y_{\bm s, a}   \\
     & \hspace{-0.2em} \mathrm{s.t.} \hspace{0.2em} \forall \omega: \sum_{s}{\pi_i(s_i|\omega)} = 1 \hspace{0.5em} \text{and} \hspace{0.5em} \forall s_i, \omega: \pi_i(s_i|\omega) \geq 0\\
    & \hspace{-0.2em} \forall \bm{s}: \sum_{a\in \A} y_{\bm s, a} = 1 \hspace{0.5em} \text{and} \hspace{0.5em} \forall \bm s, a \in \A : y_{\bm s, a} \in [0, 1] \\
    & \hspace{-0.2em} \forall \bm{s}, a' : \sum_{\omega \in \Omega \atop a \in \A} \mu_0(\omega) \bm \pi_{-i}(\bm s_{-i}|\omega) \pi_i(s_i|\omega) \Delta v(\omega, a, a')y_{\bm s, a} \ge 0.
\end{align*}
\end{restatable}

The $y_{\bm s, a} \in\{0, 1\}$ in the above program means whether the receiver takes action $a$ given joint signal $\bm s$, which can be relaxed to the continuou range $[0, 1]$. 
Briefly, the program above takes inspiration from the persuasion setting with (1) a single sender and (2) $|\Sig| = |\A|$, where signals can be interpreted as an \emph{action recommendation} and optimal signaling expressed as a linear program with an incentive compatibility constraint to ensure that the receiver follows the recommended action. In our setting, even if $|\Sig| = |\A|$, the receiver observes joint signals of size $|\Sig|^n$ and a single sender cannot unilaterally specify the joint scheme; only the marginal. Correspondingly, the program needs to resolve the action taken by the receiver and becomes a bi-linear optimization problem.

We next show that the best-response problem is NP-Hard, even with just two senders.
This means that the above bi-linear program is not computationally tractable. 
This is a key result in our work and rules out even additively approximating to the best-response in polynomial time.

\begin{theorem}[NP-hardness of Best Response]\label{thm:best-response-NP-hard}
It is NP-hard to solve the best-response problem in multi-sender persuasion, even with additive approximation error $\frac{1}{|\Omega|^6}$ and only $n=2$ senders (while $|\Omega|$ and $|\A|$ are large). 
\end{theorem}

The proof (in Appendix~\ref{app:best-response-NP-hard}) is technical and based on a non-trivial reduction from the NP-hard problem \emph{public persuasion with multiple receivers} \cite{dughmi_algorithmic_2017}. Intuitively, each signal $\bm s_{-i}$ from other non-responding senders induces a different belief about the state. From the best-responding sender's perspective, this can be correspondingly interpreted as facing multiple receivers with different prior beliefs and needing to design a single signaling scheme $\pi_i$ for all of them. With carefully crafted utilities, this problem can encode the public persuasion problem \cite{dughmi_algorithmic_2017}, which involves multiple receivers with the same belief but different utility functions.  
Our proof formally establishes this connection, which leads to the NP-hardness of our problem. 

The NP-hardness of computing best response, however, does not imply the hardness of \emph{equilibrium verification}: i.e., determining whether a given strategy profile of the senders constitutes a Nash equilibrium.  We conjecture that the equilibrium verification problem is Co-NP hard, whose formal proof would be an intriguing direction for future work. 

\begin{figure*}
    \centering
    \includegraphics[width=\linewidth]{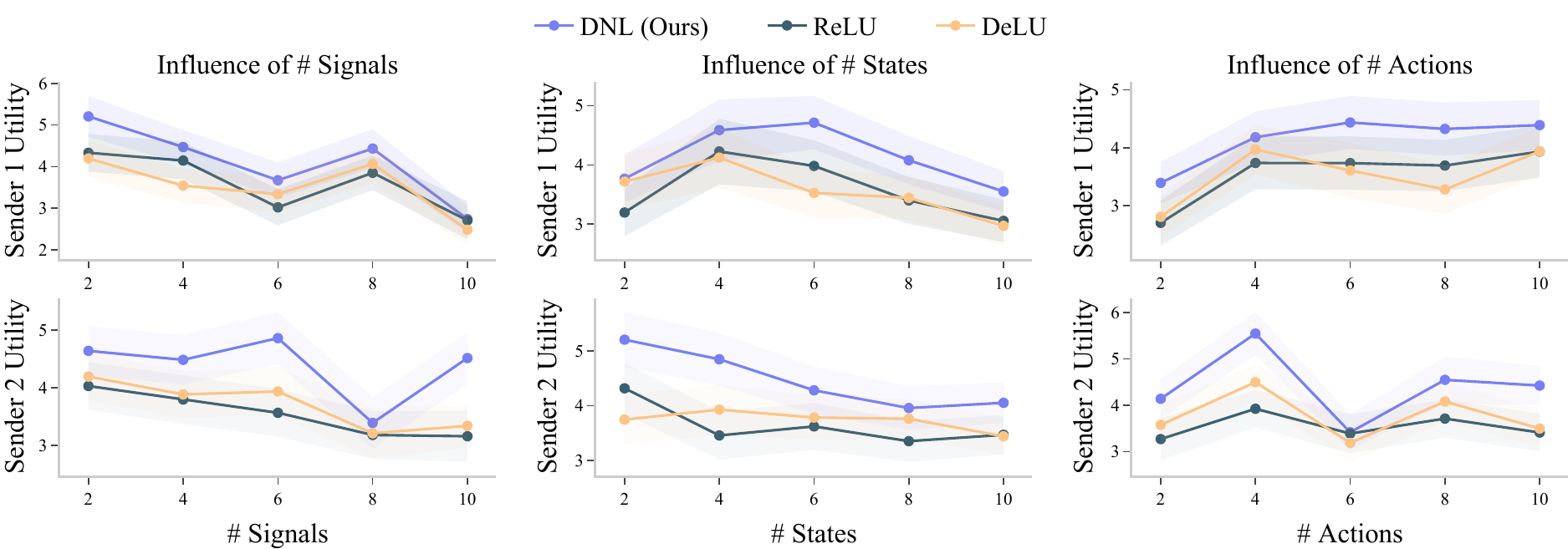}
    \vspace{-2em}
    \caption{Our method finds better $\epsilon$-local Nash equilibrium than the baseline DeLU~\cite{wang2023deep} and ReLU networks.}
    \vspace{-1em}
    \label{fig:2_sender_eq}
\end{figure*}

\subsection{Equilibrium Characterization}
A simple observation from previous works on multi-sender Bayesian persuasion \cite{gentzkow_bayesian_2017, ravindran_competing_2022} is, if for every state $\omega \in \Omega$ there is a unique optimal action for the receiver, then a simple equilibrium can be achieved by all senders fully revealing the state - i.e. $\Sig = \Omega$ and $\pi_i(s_i = \omega | \omega) = 1, \forall i$. 
Observe that this reveals the exact state realization to the receiver and thus no sender can unilaterally affect the receiver's belief and thus their action. However, for this equilibrium to exist, every sender's signal space must be as large as the state space ($|\Sig| \ge |\Omega|$), which is impractical if there are many states. Theorem \ref{thm:full-revelation} relaxes this assumption, and shows that an equivalent equilibrium exists under a much weaker assumption, $|\Sig| \geq \min(|\A|^{\tfrac{1}{n-1}}, |\Omega|^{\tfrac{1}{n-1}})$, which can be easily satisfied when there are many senders.  The proof of the theorem (in App.~\ref{app:full_reveal_proof}) is constructive and builds a mapping between signals and actions inspired by grey codes \citep{wilf1989combinatorial}. 

\begin{theorem}[Full-Revelation Equilibrium]
\label{thm:full-revelation}
    Suppose $|\Sig| \geq \text{min}(|\Theta|^{1/(n-1)}|, \A|^{1/(n-1)}$, and for every state $\omega\in\Omega$ there is a unique optimal action for the receiver. Then, the multi-sender persuasion game has an NE that fully reveals the optimal action for the realized state to the receiver. This equilibrium, however, is not necessarily unique. 
\end{theorem}

While the result above generalizes an explicit equilibrium characterization to a much larger setting with limited signals, the corresponding equilibrium is optimal for the receiver and not necessarily the senders. Indeed, if the preferences of the senders do not perfectly align with the receiver (which is common and in-fact the premise behind persuasion), this will not be beneficial for the senders. Further, the construction above is based on the assumption that for every state, the receiver has a unique optimal action. This may not hold in many scenarios, with receivers being indifferent between multiple actions. We show that in such scenarios and thus the general case, finding an equilibrium is PPAD-Hard, even with constant number of senders, states, and actions. The proof (in App.~\ref{app:PPAD-non-fixed}) relies on a reduction from finding equilibrium in two-player games with binary utilities. 

\begin{theorem}[PPAD-Hardness]\label{thm:PPAD-non-fixed}
    In multi-sender persuasion games that do not satisfy the condition ``the receiver has a unique optimal action for every state $\omega$'', under some tie-breaking rules, finding NE is PPAD-hard. This holds even if $n=2$, $|\Omega|=2$, $|\A|=4$ (while $|\Sig|$ is large).
\end{theorem}

\section{Deep Learning for Local Equilibrium}\label{sec:deep}

The strong computational hardness results established in the previous section 
motivate us to find methods to efficiently calculate \emph{local} Nash equilibrium, a strategy profile wherein any small unilateral deviation cannot improve a player's utility. 
This has been promoted as an attractive solution concept for a plethora of settings~\cite{fiez2020implicit,fiez2021global,jin2020local}. In doing so, we also relax the assumption of having access to the exact utility model and take a sample-based approach popularized in the nascent literature on differentiable economics. This is especially prescient as it gracefully generalizes to settings where action space is rich (or even continuous) and it may only be possible to sample utilities for arbitrary policies. Correspondingly, we introduce a computational framework based on deep learning. It consists of a novel discontinuous neural network architecture approximating the senders' utility functions and a local equilibrium solver running extra-gradients on the learned discontinuous networks. We describe the learning framework in detail and compare the found local NE against those obtained by strong baseline network structures as well as the full revelation solution (Theorem~\ref{thm:full-revelation}).

\begin{definition} [$\epsilon$-Local Nash Equilibrium]
    An $\epsilon$-local Nash equilibrium for a multi-sender persuasion game is a profile of signaling policies $\bpi = (\pi_1, \ldots, \pi_n)$ such that for any sender $j$ and deviating policy $\pi_j'\in\{\pi'\ |\ \|\pi'-\pi_j\|\le \epsilon\}$, it holds that $\ubar_j(\bpi) \ge \ubar_j(\pi'_j, \bpi_{-j})$.
\end{definition}

\subsection{Method}
We aim to use the extra-gradient~\cite{korpelevich1976extragradient} method
to find an $\epsilon$-local NE. However, the major challenge in applying this, or indeed any other gradient-based learning algorithm, is that the senders' utility function is discontinuous and non-differentiable in their signaling policy, as per Proposition \ref{prop:discontinious}. Conventional neural networks well approximate continuous functions but are not expressive enough to express discontinuous functions~\cite{scarselli1998universal}. To solve this problem, we extend a fully connected feedforward network with ReLU activation~\cite{agarap2018deep} to learn a differentiable representation of discontinuous functions. To describe our method, we first introduce the activation pattern and the piecewise linearity of ReLU networks.

\textbf{ReLU networks}\ \ Suppose there are $L$ hidden layers. Layer $l$ has weights $\mW^{(l)}\in\mathbb{R}^{n_l\times n_{l-1}}$ and biases $\vb^{(l)}\in\mathbb{R}^{n_l}$. $n_0 = d$ is the input dimension. The output layer has weights $\mW^{(L+1)}\in\mathbb{R}^{d'\times n_L}$ and  biases $\vb^{(L+1)}\in\mathbb{R}^{d'}$. With input $\vx\in \mathbb{R}^d$, we have the pre- and post-activation output of layer $l$: $\vh^{(l)}(\vx)=\mW^{(l)}\vo^{(l-1)}(\vx)+\vb^{(l)}$ and $\vo^{(l)}(\vx)=\sigma\left(\vh^{(l)}(\vx)\right)$, where $\sigma(x)=\max\{x,0\}$ is the ReLU activation. For each hidden unit,  the ReLU {\em activation status} has two values, defined as $1$ when pre-activation $h$ is positive and $0$ when $h$ is strictly negative. The activation pattern of the entire network is defined as follows.
\begin{definition}[Activation Pattern]
    An {\em activation pattern} of a ReLU network is a binary vector $\vr=[\vr^{(1)},\cdots,\vr^{(L)}]\in\{0,1\}^{\sum_{l=1}^L n_l}$, where $\vr^{(l)}$ is a 
    {\em layer activation pattern} including the activation status of each unit in layer $l$.
\end{definition}
The activation pattern depends on the input $\vx$. Given an activation pattern $\vr(\vx)$, the ReLU network is a linear function~\cite{croce2019provable}
\begin{align}
    \vh^{(L\shortp 1)}(\vx)=\mM^{(L\shortp 1)}\vx+\vz^{(L\shortp 1)},\nonumber
\end{align}
where $\mM^{(L\shortp 1)}=\mW^{(L\shortp 1)} (\prod_{k=1}^{L}\mR^{(L\shortp 1-k)}(\vx)\mW^{(L\shortp 1-k)})$, $\vz^{(L\shortp 1)}=\vb^{(L\shortp 1)}+\sum_{k=1}^{L}(\prod_{j=0}^{L-k}\mW^{(L\shortp 1-j)}\mR^{(L-j)}(\vx))\vb^{(k)}$, and $\mR^{(k)}$ is a diagonal matrix with diagonal elements equal to the layer $k$'s activation pattern $\vr^{(k)}$. 

\begin{figure*}
    \centering
    \includegraphics[width=\linewidth]{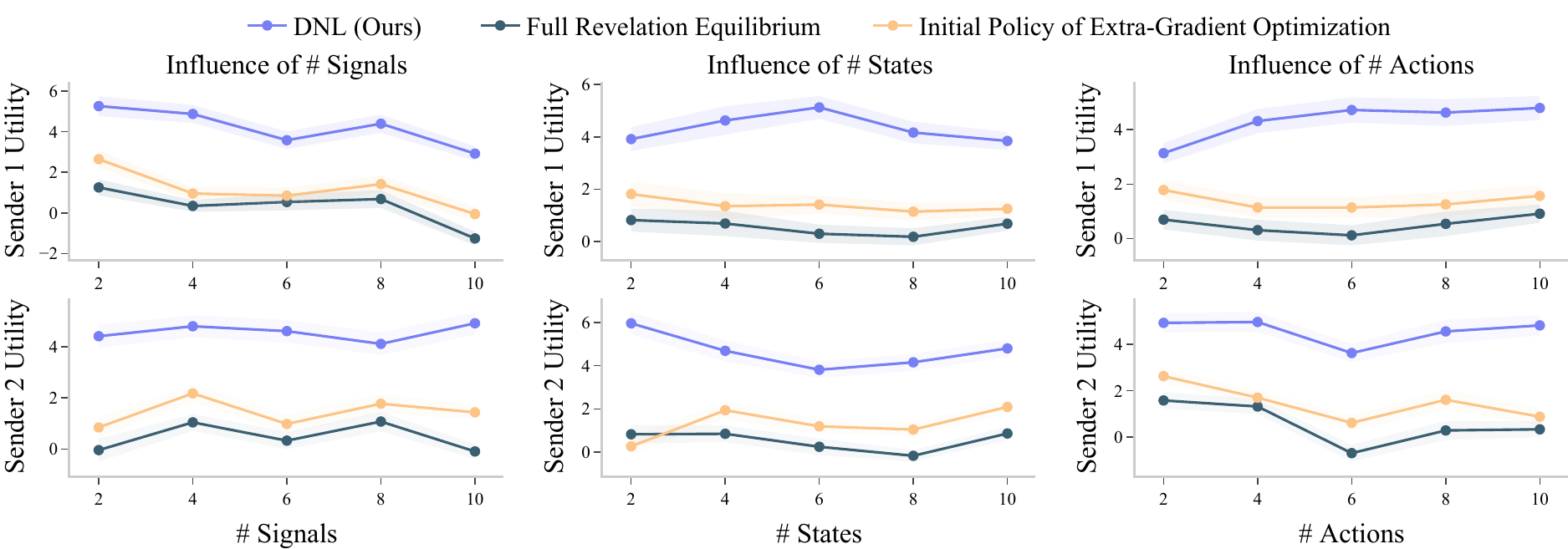}
    \vspace{-2em}
    \caption{The $\epsilon$-local Nash equilibria found by our method typically Pareto dominate the full revelation equilibria and improve the random initial policies of extra-gradient by a large margin.}
    \vspace{-1em}
    \label{fig:2_sender_eq_ab}
\end{figure*}
\textbf{Previous work}\ \  To introduce discontinuity, DeLU~\cite{wang2023deep} proposes to generate the bias of the last layer $\vb^{(L+1)}$ by an auxiliary network that is conditioned on the activation pattern $\vr(\vx)$. The idea is that inputs with the same $\vr(\vx)$ come from a polytope that is the intersection of half-spaces: $\mathcal{D}(\vx) = \cap_{l=1,\cdots,L}\cap_{i=1,\cdots,n_l}\Gamma_{l,i}$, where $\Gamma_{l,i}$ corresponding to unit $i$ of layer $l$ defined as:
\begin{align}
\Gamma_{l,i}=\left\{\vy\in\mathbb{R}^d | \Delta_i^{(l)}\left(\mM_i^{(l)}\vy+\vz_i^{(l)}\right)\ge0\right\}.\label{equ:sub-spaces}
\end{align} 
Here $\mM_i^{(l)}\vy+\vz_i^{(l)}$ is the output of unit $i$ at layer $l$, and $\Delta_i^{(l)}$ is 1 if $\vh^{(l)}_i(\vx)$ is positive, and is -1 otherwise.

In this way, different pieces $\mathcal{D}(\vx)$ has different biases, introducing discontinuity at piece boundaries. However, since inputs in the same piece share the same weights, DeLU is a linear function in a piece and does not have enough expressivity to represent the utility function in the multi-sender persuasion games, which is piecewise non-linear (Proposition~\ref{prop:discontinious}).

\begin{figure*}
    \centering
    \includegraphics[width=\linewidth]{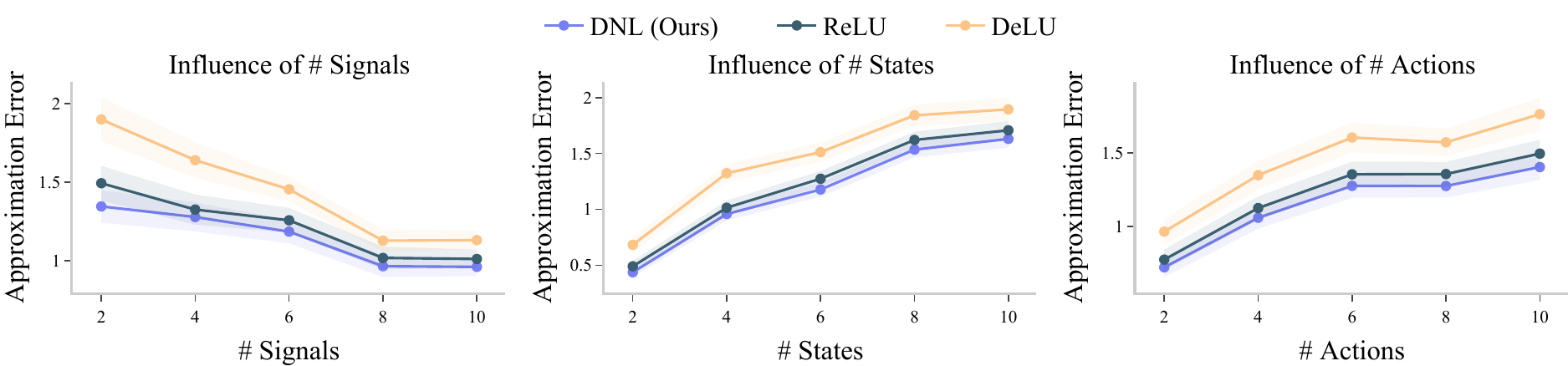}
    \vspace{-2em}
    \caption{Our network achieves lower approximation errors compared to baseline network structures.}
    \vspace{-1em}
    \label{fig:2_sender_error}
\end{figure*}
\textbf{Network architecture}\ \  We enable a fully-connected network to be piecewise Discontinuous and Non-Linear (\name) by dividing the network into a lower part and a higher part. The lower part consists of the first $K<L$ linear layers and is a normal network with ReLU activation. During a forward pass, we get the activation pattern
\begin{align}
\vr^{(\le K)}=[\vr^{(1)},\cdots,\vr^{(K)}]\nonumber
\end{align}
of this lower network and generate the weights and biases of the higher part via a hyper-network $g$ whose input is $\vr^{(\le K)}$.

Looking at the lower part, inputs with the same $\vr^{(\le K)}$ reside in the intersection of half-spaces:
\begin{align}
\mathcal{D}^{(\le K)}(\vx) = \cap_{l=1,\cdots,K}\cap_{i=1,\cdots,n_l}\Gamma_{l,i},\nonumber
\end{align}
with $\Gamma_{l,i}$ defined in Eq.~\ref{equ:sub-spaces}. By introducing the hyper-network, inputs in $\mathcal{D}^{(\le K)}(\vx)$ share a non-linear higher network. Therefore, within this piece, the utility approximation can be non-linear. Furthermore, different pieces have different $\vr^{(\le K)}$, so the higher part can be different, introducing discontinuity at boundaries.

Formally, we train a network $f_j(\bm\pi; \theta_j)$, parameterized by $\theta_j$, for each sender $j$ to approximate its ex-ante utility (Definition~\ref{def:ex-ante_u}) under the joint signaling policy $\bpi$. The input to the lower part of $f_j$ is the joint signaling policy $\bm\pi$. The hyper-network $g$ takes the activation pattern $\vr^{(\le K)}(\vx)$ of the lower part as input and outputs $\{(\mW^{k}, \vb^{k})\}_{k=K+1}^{L+1}$ as the weights and biases for layer $K+1$ to $L+1$. After obtaining its weights and biases, the higher part then takes the output of the lower part network as input and generates an approximation of the ex-ante utility. It is worth noting that $K<L$, and we have at least two linear layers at the higher part, so that piecewise non-linearity can be ensured. The whole network $f_j(\bm\pi; \theta_j)$ is end-to-end differentiable and updated by the MSE loss function:
\begin{align}
    \mathcal L(\theta_j) = \mathbb{E}_{\bpi} \left([f_j(\bpi;\theta_j) - \ubar_j(\bpi)]^2\right).\label{eq:learning_loss}
\end{align} 
To calculate this loss, we uniformly sample joint policies $\bpi$ and obtain the corresponding ex-ante utility $\ubar_j(\bpi)$ by running a game simulator.

\textbf{Extra-gradient}\ \  With $f_j$ as a differentiable representation of the senders' ex-ante utility, we can run extra-gradients to find $\epsilon$-local NE. We directly parameterize the signaling policy $\pi_j$ of sender $j$ by a learnable matrix $\phi_j$ residing in $\Phi\subset \mathbb{R}^{|\Omega|\times|\mathcal S|}$. A matrix in $\Phi$ has all of its elements in the range $[0,1]$, and each row summed to 1. 

The extra-gradient update can be written as
\begin{align}
    & \text{(extrapolation)}\ \phi_j^{\tau+1/2} = p_\Phi(\phi_j^{\tau}-\gamma_\tau \nabla_{\phi_j^\tau}f_j(\bpi_{\phi^{\tau}};\theta_j)),  \nonumber \\
    & \text{(update)}\ \ \phi_j^{\tau+1} = p_\Phi(\phi_j^{\tau}-\gamma_\tau \nabla_{\phi_j^{\tau+1/2}}f_j (\bpi_{\phi^{\tau+1/2}};\theta_j)).  \nonumber
\end{align}
Here, $p_\Phi[\cdot]$ is the projection to the constraint set $\Phi$, and we use a SoftMax projection in practice. The parameters $\theta_j$ of $f_j$ is fixed during extra-gradient updates. $\bpi_\phi$ is the joint parameterized signaling policy, and $\gamma_\tau$ is the learning rate.
\section{Empirical Results}\label{sec:experiment}

In this section, we evaluate our deep learning method by comparing against continuous neural networks with ReLU activation, discontinuous neural networks DeLU~\cite{wang2023deep}, and full-revelation strategies on a illustrative example and a synthetic benchmark.

\subsection{Didactic Example}
We first demonstrate the representational capacity of our method on a simple multi-sender game with 2 senders, 2 signals, 2 actions, and 2 states. The utility matrix of the receiver is \( \left[\begin{smallmatrix} 1 & 0 \\ 0 & 1 \end{smallmatrix}\right] \), where each row corresponds to a state and each column corresponds to an action. The utilities for two senders are \( \left[\begin{smallmatrix} 1 & 1 \\ -1 & 3 \end{smallmatrix}\right] \) and \( \left[\begin{smallmatrix} 4 & 1 \\ 1 & 1 \end{smallmatrix}\right] \), respectively.

\begin{figure*}
    \centering
    \includegraphics[width=\linewidth]{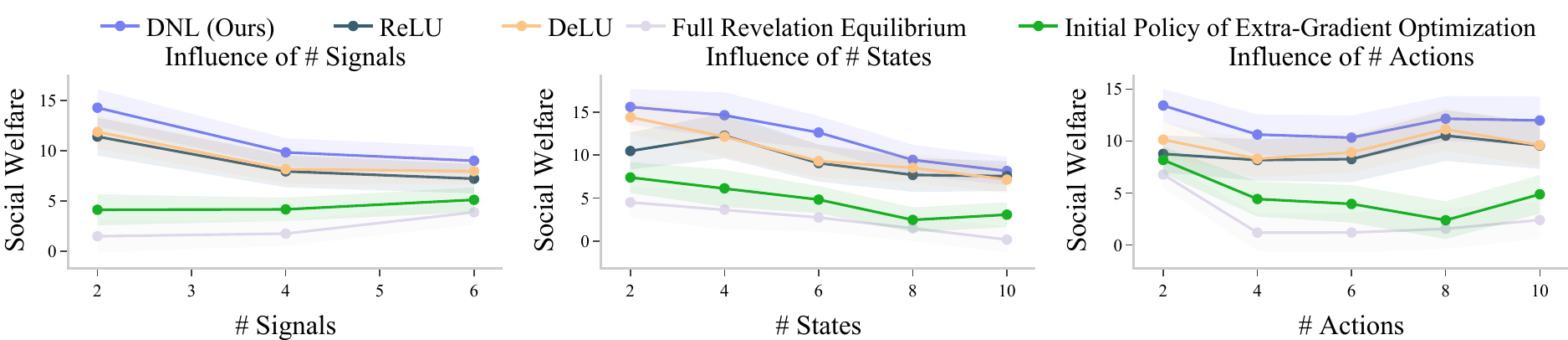}
    \vspace{-2em}
    \caption{Our method achieves higher social welfare compared against baselines and full-revelation solutions in games with 4 senders.}
    \label{fig:4_sender}
    \vspace{-1em}
\end{figure*}
In the first row of Fig.~\ref{fig:viz}, we fix both of the two senders' signaling policies at $\mathtt{State\ 2}$ to $(0.5, 0.5)$ and vary their signaling policies at $\mathtt{State\ 1}$. The x-axis is the probability of $\mathtt{Sender\ 1}$ sending $\mathtt{Signal\ 1}$ at $\mathtt{State\ 1}$, the y-axis is the probability of $\mathtt{Sender\ 2}$ sending $\mathtt{Signal\ 1}$ at $\mathtt{State\ 1}$, and the z-axis is the (possibly approximated) ex-ante utility of $\mathtt{Sender\ 2}$. The second row is similar to the first, but the two senders' signaling policies at $\mathtt{State\ 2}$ are $(0.2, 0.8)$ and $(0.8, 0.2)$, respectively.


The first column shows $\mathtt{Sender\ 2}$'s actual ex-ante utility. This utility function displays discontinuities, effectively captured by our method (second column). In contrast, ReLU approximations in the third column are not accurate at piece boundaries, and we can observe that the approximated DeLU function in the fourth column is linear in each piece, limiting its representational power for this game. 

To ensure a fair comparison, networks used in this study, including ours, ReLU, and DeLU, are standardized in terms of architecture, featuring three hidden layers with 64 units each. The training process involves a dataset of 500,000 randomly selected samples (pairs of signaling policies and corresponding ex-ante utilities), over which the networks are trained for a total of 200 epochs. For our network, the lower part has the first hidden layer. This layer's activation pattern is used to generate the weights and biases of the subsequent two layers by a hyper-network, which is itself composed of two hidden layers, each containing 32 units.

\subsection{Synthetic Benchmark}
In this section, we generate synthetic problems to test whether our network can find $\epsilon$-local Nash equilibria that are better (in terms of sender utility) than those found by baseline network architectures as well as the full-revelation equilibria mentioned in Theorem \ref{thm:full-revelation}.

\textbf{Setup}\ \  The size of a problem is determined by the tuple $(n, |\Omega|, |\mathcal S|, |\mathcal A|)$, and our evaluation encompasses a range of problem sizes to thoroughly assess the efficacy of our method. Specifically, we consider 2 and 4 senders, and for each, $(|\Omega|, |\mathcal S|, |\mathcal A|)$ are drawn from a three-dimensional grid $\{2, 4, 6, 8, 10\}^3$. For each problem size, we randomly generate 5 problem instances. In total, we have 1,250 problem instances to benchmark the proposed learning framework. The utility matrices for the receiver and senders, as well as the prior belief of states, are randomly sampled from a Gaussian distribution with variance at 100 and mean at 0, with a SoftMax applied to generate prior beliefs. We employ random numbers featuring significant variance to enhance the complexity of the benchmark, thereby facilitating a more effective evaluation of different solutions.

We standardize the network architecture of our method and baselines mirroring the configuration delineated in the didactic example to ensure a fair comparative analysis. The training setup is described in detail in Appendix~\ref{appx:deep}. To test whether a joint signaling policy profile $\bpi$ is an $\epsilon$-local Nash equilibrium, we randomly sample $K$ policies $\bpi'_j$ for each sender $j$ in the neighborhood $\{\pi'_j\ |\ \|\pi'_j-\pi_{\phi_j}\|_\infty\le \epsilon\}$ and check whether it can gain a higher utility at $\bpi'_j$. In our experiments, the number of test samples $K$ grows with the problem size. We set the neighborhood size $\epsilon$ to 0.005 and find that our experimental results are robust with the value of $\epsilon$ up to 0.01 as evidenced by more results in Appendix~\ref{appx:deep}.

\textbf{Representational capacity}\ \  In Fig.~\ref{fig:2_sender_error}, we fix the number of senders to 2 and compare the approximation errors (Eq.~\ref{eq:learning_loss}) achieved by our method and the two baseline architectures. We show the influence of the numbers of signals, states, and actions in three subplots, respectively, by presenting the average (solid lines) and the 95\% confidence interval (shaded areas) of approximation errors. In the first subplot for example, we iterate the number of signals, and present results on all problem instances for each number of signals.

The results suggest that our algorithm provides a more accurate approximation than ReLU and DeLU. The advantage of our method is consistently maintained across all the range evaluated. It is also interesting to observe that the approximation error decreases for all three algorithms as the number of signals increases, but it increases as the numbers of states and actions increase. This observation indicates that the multi-sender persuasion game becomes more challenging with fewer signals, aligning with existing theoretical results on persuasion with limited signals \cite{dughmi2016persuasion}. 

\textbf{Equilibrium}\ \  In Fig.~\ref{fig:2_sender_eq}, we conduct a comparison between the equilibrium derived from our method against those produced by baselines. We run the verification process to ascertain whether the extra-gradient outcomes are indeed $\epsilon$-local NE. We present the mean and the 95\% confidence interval of the sender utilities at the best solutions that satisfy the criteria. Notably, our findings indicate that for each of the two senders, the ex-ante utility achieved in our model consistently outperforms that of the baselines, exhibiting Pareto dominance. In Fig.~\ref{fig:2_sender_eq_ab}, we provide additional evidence demonstrating that extra-gradient with our trained networks can significantly enhance the senders' utility from the initial starting points. Furthermore, \name~successfully generates solutions that surpass full-revelation equilibria by a large margin. This improvement underscores the synergisitic benefit of integrating the extra-gradient approach with our networks. Similar results can be observed for games with 4 senders, and in Fig.~\ref{fig:4_sender}, we show the welfare (the sum of senders' utilities) in these games.

\subsection{Real-World Scenarios}

\textbf{Setting}\ \ In this section, we extend the evaluation of our method to the following real-world scenarios.

\textbf{\emph{Scenario 1: Advertising\ of\ Quality}}\ \ Prior economic research on multi-sender persuasion explored an advertising problem~\citep{gentzkow2017bayesian}. In this problem, a total of $n$ competing firms (senders) market their products to a single consumer. The product of each firm $i$ can be of high quality ($\omega_i = 5$) or low quality ($\omega_i = -5$). The consumer wants to buy at most one product. The quality of the products is the state, known to firms but not the consumer. By sending signals, i.e., verifiable advertisements about their product's quality, a firm tries to persuade the consumer into purchasing from it, which induces utility 1 for the firm. The firm's utility is 0 if the consumer doesn't purchase from it. The consumer is faced with $n+1$ actions, purchasing from any one of the firms, or none at all. The consumer's utility of purchasing from firm $i$ is $\omega_i+\epsilon_i$, where $\epsilon_i$ is a shock (Gaussian-distributed zero-mean noise). If the consumer makes no purchase, their utility is $0$.
In our experiments, we set $n$ to 7 and generate 20 instances randomly. 

\textbf{\emph{Scenario 2: Advertising\ of\ Multiple\ Products}}\ \ We make the previous advertising example more realistic by incorporating multiple products of different quality and prices. Specifically, we consider the following problem.

There are $n$ firms (senders), each of which $i$ sells a product of price $p_i$ and quality $\omega_i$. The true state is the prices and quality of all products. The consumer (receiver) has a partial observation of the state, as it has no access to the quality of products. The receiver has $n+1$ actions, which are buying a product from one of the firms or buying nothing. The utility of firm $i$ is $p_i$ if the receiver buys from it, or -1 otherwise. The senders use signals to strategically reveal the quality information to the receiver, trying to sway their purchase decisions in their favor. The receiver wants to maximize its utility, which is $\omega_i-p_i+\epsilon_i$ if purchasing product $i$, or 0 if buying nothing. Here $\epsilon_i$ is the shock defined in the same way as in Scenario 1. We test the case with $n=2$ firms. Price $p_i$ and quality $\omega_i$ are uniformly random integers in the range [1, 10] and [-8, 12], respectively. 

\textbf{\emph{Scenario 3: Uber\ or\ Lyft}}\ \ In this last scenario, we move beyond advertising and consider the competition among real-world ride-hailing apps, and a single driver subscribed to both platforms. There are two senders, Uber and Lyft, who receive $m$ and $n$ orders from users, respectively. Each order has four features (1) The price charged to the user; (2) The payment to the driver; (3) The true utility to the app, which is the price minus the payment; and (4) The true cost for the driver, which is known to the app and is influenced by many factors, such as the user rating indicating whether they are friendly, the expected travel time and distance, the expected waiting time, etc. 

The true state is the joint feature of all orders. Feature (4), the true cost to the driver, is invisible to the driver when they must decide the pickup. Uber and Lyft can send signals to strategically reveal this information in order to persuade the driver into picking up their orders. The driver has $m+n+1$ actions, which are picking up one of the $m+n$ orders or doing nothing. The utility of the driver is the price minus the true cost of the selected order, or -1 if they don’t select any order. In our experiments, we set $m$ and $n$ to 4 and the number of signals to $m+1$. 

\textbf{Results}\ \ We test the performance of our method and the baseline neural network structures. Table~\ref{tab:real} shows the social welfare of the senders at the $\epsilon$-local equilibria found by different methods. Mean and 95\% confidence intervals with 20 random instances are presented. We can observe that our method consistently outperforms other methods, indicating that our discontinuous, piecewise nonlinear network structure allows us to effectively tackle these richer settings that prior literature could not.

\begin{table}
    \centering
    \setlength{\tabcolsep}{5pt}
    \caption{The social welfare (avg±95\% confidence interval) at the $\epsilon$-local equilibria found by our method and baseline networks.}\label{tab:real}
    \vspace{0.2em}
    \begin{tabular}{cccc}
    \toprule
       Scenario  & ReLU & DeLU & Ours \\
    \hline
       1  & 0.498$\pm$0.004 & 0.599$\pm$0.003 & \textbf{0.699$\pm$0.003} \\
       2  & 0.407$\pm$0.176 & 0.467$\pm$0.179 & \textbf{0.526$\pm$0.004} \\
       3  & 3.216$\pm$0.790 & 3.783$\pm$0.894 & \textbf{4.344$\pm$0.885} \\
    \bottomrule
    \end{tabular}
    \vspace{-2em}
\end{table}
\section{Discussion}\label{sec:discussion}
We provide a comprehensive computational study of multi-sender Bayesian persuasion, a model for a wide range of real-world phenomena. 
The complex interplay of simultaneous sender actions and sequential receiver responses makes this game challenging. Our work formalizes this challenge by proving computational hardness results for both best response and equilibrium computation. Relaxing the equilibrium concept, however, offers hope, even without complete information. We propose a novel class of neural networks that can approximate the non-linear, discontinuous utilities in this game; paired with the extra-gradient algorithm, it is highly effective at finding local equilibria. Indeed, our network may be of broader interest to many games with discontinuous utility as it facilitates any downstream optimization algorithm. More broadly, BP is part of the principal-agent model of economics which also includes problems like contract design and Stackelberg games. Insights developed here can be instrumental to multi-principal variants of those problems which, despite their importance, have long eluded robust computational solutions. 

\newpage

\section*{Acknowledgements}
We thank the reviewers for the helpful comments.  This work is supported by NSF awards No. IIS-2147187 and No.~CCF-2303372, Army Research Office Award No.~W911NF-23-1-0030 and Office of Naval Research Award No.~N00014-23-1-2802.

\section*{Impact Statement}
This paper presents work whose goal is to advance the field of Machine Learning. There are many potential societal consequences of our work, none of which we feel must be specifically highlighted here.


\bibliography{bibliography}
\bibliographystyle{icml2024}

\newpage
\appendix
\onecolumn

\section{Proofs}
\subsection{Proof of Proposition~\ref{prop:discontinious}}\label{app:discont_prop}
\begin{proof}
    For a joint scheme $\bm{\pi}$, each signal realization $\bm{s}$ induces a posterior belief $\mu_{\bm{s}}$, wherein receiver take optimal action $a^*(\mu_{\bm{s}})$. We can equivalently write the function $a^*$ in terms $\bm{\pi}(s|\cdot)$, ad note that $a^*(\bm{\pi}(s|\cdot)) \in \A$. When the signaling scheme changes sufficiently such that the new actions are optimal for a given realized posterior $\mu_{\bm{s}}$, the mapping $a^*(\bm{\pi}(\bm{s}|\cdot))$ changes accordingly. Thus, the function $a^*(\bm{\pi}(s|\cdot))$ is piece-wise constant with the boundary between pieces representing this changed mapping. The utility of sender $i$ is given by:
    \begin{gather}
         \sum_{\omega \in \Omega} \sum_{\bm s \in \Sig^n} \mu_0(\omega) u_i(\omega, a^*(\bm{\pi}(\bm{s}|\omega))) \bm{\pi}(\bm{s}|\omega) \\
         \sum_{\omega \in \Omega} \sum_{\bm s \in \Sig^n} \mu_0(\omega) u_i(\omega, a^*(\bm{\pi}(\bm{s}|\omega))) \prod_{i}{\pi(s_i|\omega)}
    \end{gather}
    where we note that since $u_i$ is essentially indexing a matrix, $u_i(\omega, a^*(\bm{\pi}(\bm{s}|\omega)))$ is piece-wise constant with the same boundaries as $a^*(\bm{\pi}(\bm{s}|\cdot))$. It is evident from the last expression above the utility is piece-wise bi-linear in $(\pi_1, \dots, \pi_n)$ and upon fixing $\bm{\pi}_{-i}$ is it piecewise linear in $\pi_i$. 
\end{proof}

\subsection{Proof of Proposition~\ref{prop:best-response-bi-linear}}\label{app:best_resp_bi_linear}
\begin{proof}
    Note that $\bm{\pi}_{-i}$ refer to the signaling of others and is fixed, with the optimization variables being $\pi_i$ and $y_{\bm{s}, a}$. Next, observe that if $y_{\bm{s},a} \in \{0, 1\}$ then this optimization program can be interpreted as follows. $\pi_i$ denotes the signaling scheme of influence $i$, and $y_{\bm{s},a}$ denotes whether action $a$ is the optimal action for the user upon receiving the joint signal $\bm{s}$ and computing the corresponding posterior belief. The sum constraint on $y_{\bm{s},a}$ ensure $[y_{\bm{s},1}, \dots, y_{\bm{s},|\A|}]$ is a one hot vector. To ensure that the choice of $y_{\bm{s},a}$ are indeed correct, we need to ensure incentive-compatible. That is, we require the following holds for the posterior induced by any joint signal $\bs$, and any action $a'$: $\sum_{\omega}{P(\omega|\bs)\sum_{a \in \A}\left[v(a, \omega) - v(a', \omega)\right]y_{\bm{s}, a}}$. By Bayes rule, $P(\omega|\bs) = \frac{\bpi(\bs|\omega)\mu_0(\omega)}{P(\bs)}$ and since $P(\bs)$ is constant for the whole sum, we can multiply both sides by $P(\bs)$ and arrive at the first constraint in the above LP. Since this constraint enforces the choice of user action at each posterior indeed correct, the objective simply maximizes the sender's ex-ante expected utility.

    The only difference between the presented optimization problem and the best-response sketched above is that the variables $y_{\bm{s},a}$ are now relaxed to within the continuous range $[0,1]$. We now show that this relaxation does not change the optimal solution. That is, an optimal solution to the binary-constrained setting is also an optimal solution to the relaxed continuous setting. Fix any signaling scheme $\pi_i$ and any joint signal realization $\bm{s}$. Let $a^*_{\bm{s}}$ denote a best action for the user at the posterior induced by signal realization $\bm{s}$ with the schemes $(\pi_i, \bm{\pi}_{-i})$. Then we can rewrite the incentive compatibility constraint (first constraint) as follows (for brevity, we will write $\bm{\pi}(\bm{s}|\omega) = \pi_i(s_i|\omega)*\bm{\pi}_{-i}{\bm{s}_{-i}|\omega}$:
    \begin{equation}
        \sum_{a \in \A}{y_{\bm{s},a}\left\{ \sum_{\omega \in \Omega}{\mu_0(\omega) \bm{\pi}(\bm{s}|\omega) v(w, a)} - \sum_{\omega \in \Omega}{\mu_0(\omega) \bm{\pi}(\bm{s}|\omega) v(w, a^*_{\bm{s}})} \right\}} \geq 0
    \end{equation}
    Note that the first summation term inside the inner bracket is proportional to the expected utility for action $a$ under the posterior induced by $\bm{s}$, while the second summation term is the expected utility for action $a$ under this same posterior. If $a_{\bm{s}}^*$ is the unique action that maximizes expected user utility at this posterior, then the only way this can be satisfied is by setting $y_{\bm{s},a^*_{\bm{s}}} = 1$ and 0 to all others. If multiple actions may be optimal for the receiver at this belief, then let $a^*_{\bm{s}}$ be the action among these that is most preferred by sender $i$ (if there is a tie here, pick arbitrarily). Thus, by setting the corresponding $y_{\bm{s},a^*_{\bm{s}}} = 1$ and 0 for the rest satisfies the constraint while also maximizing user utility. Thus it follows that relaxing the domain of $y_{\bm{s},a}$ does not change the optimal solution since these still occur at the endpoints $0$ or $1$, and it follows that the continuous bi-linear optimization problem above corresponds to sender $i$'s best response. 
\end{proof}

\newpage
\subsection{Proof of Theorem \ref{thm:best-response-NP-hard}}\label{app:proof_np_hard}
\label{app:best-response-NP-hard}




\begin{figure}[h!]
\centering
\begin{tikzpicture}
  \node (img)  {\includegraphics[width=0.5\linewidth]{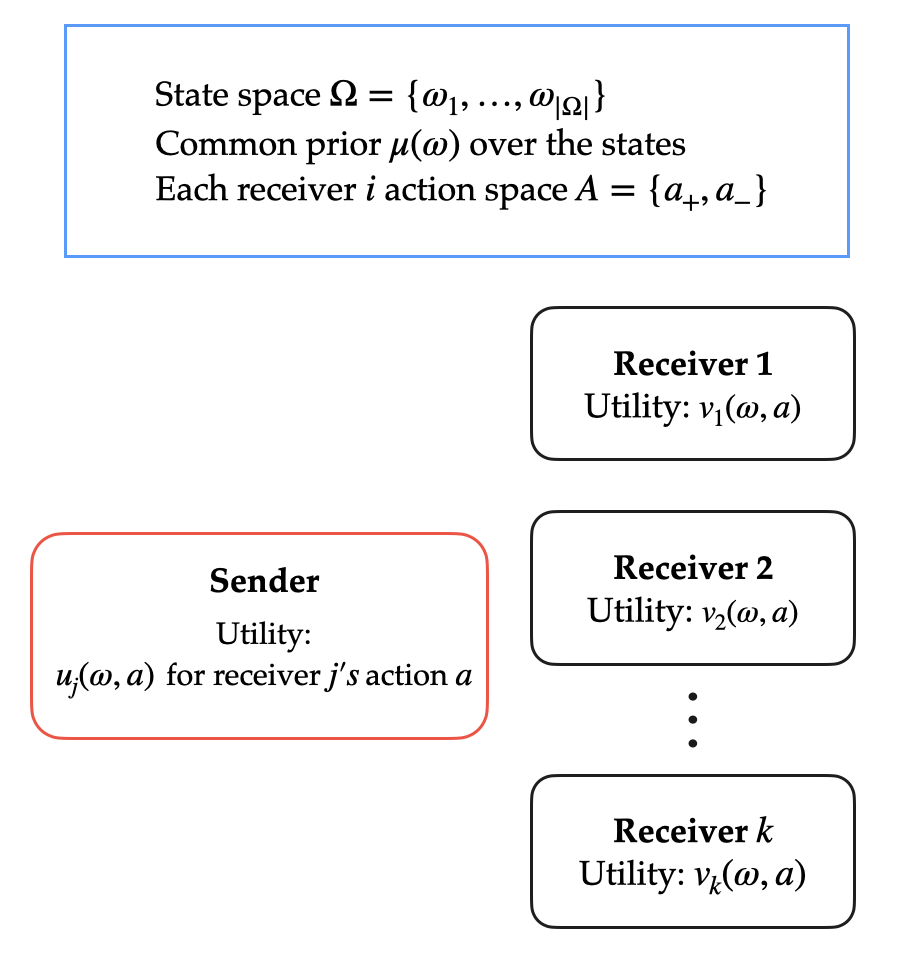}};
 \end{tikzpicture}
  \caption{Public persuasion with $k$ receivers, each with binary actions}\label{fig:np_hard_1}
\end{figure}

\begin{figure}[h!]
\centering
\begin{tikzpicture}
  \node (img)  {\includegraphics[width=0.7\linewidth]{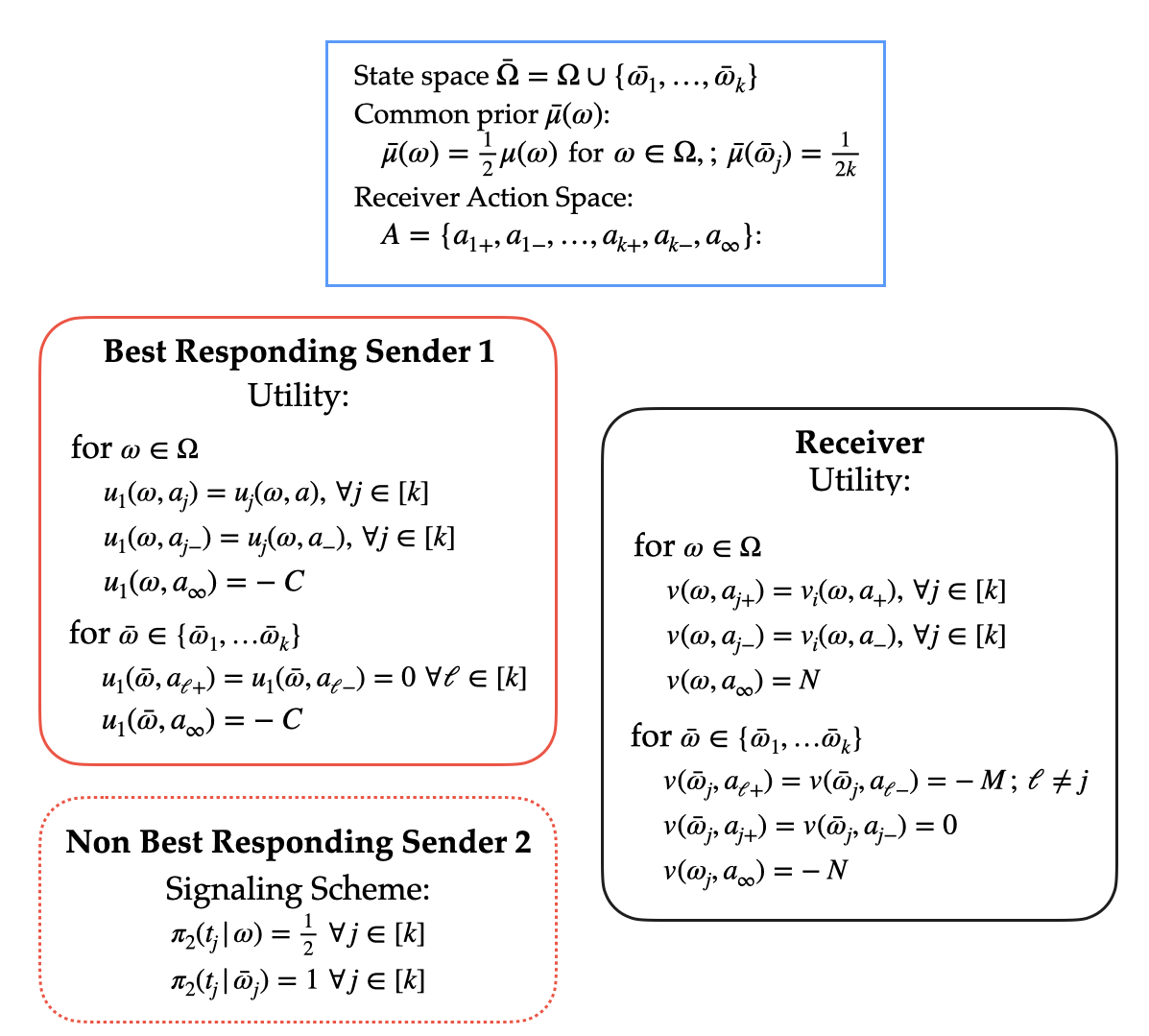}};
\end{tikzpicture}
\caption{The multi-sender persuasion problem reduced from public persuasion. $k$ additional states are added with the sole receiver having $2k+1$ actions. The receiver and best-responding sender's utility are chosen such that for all possible $k$ possible signal realization of non-best responding sender, the single receiver's plausible actions mimic that of the $k^{th}$ receiver in public persuasion.}\label{fig:np_hard_2}
\end{figure}

The proof uses a reduction from the following problem called \emph{public persuasion} \cite{dughmi_algorithmic_2017}:
\begin{definition}
A \emph{public persuasion (\textbf{Pub})} problem (with multiple receivers with binary actions) is described by tuple $\langle k, \Omega, \mu, \{ v_j(\omega), u_j(\omega) \}_{j\in [k], \omega \in \Omega} \rangle$, where: 
\begin{itemize}
    \item There are $k$ receivers denoted by $[k] = \{1, \ldots, k\}$ each having two actions $\{+, -\}$. 
    \item $\mu \in \Delta(\Omega)$ is a prior distribution of states $\omega \in \Omega$.
    \item Let $v_j(\omega, +), v_j(\omega, -) \in [0, 1]$ be the utilities of receiver $j \in [k]$ when taking actions $+$, $-$ and the state is $\omega$.  Let $v_j(\omega) = v_j(\omega, +) - v_j(\omega, -) \in [-1, 1]$ be the utility difference. 
    \item $u_j(\omega, +), u_j(\omega, -) \in [0, 1]$ are the utilities of the sender when receiver $j \in [k]$ takes action $+, -$, respectively. 
\end{itemize}
\end{definition}
Let $\pi : \Omega \to \Delta(S)$ be a signaling scheme of the sender.  Let $x_s \in \Delta(\Omega)$ denote the posterior distribution over states induced by signal $s \in S$: 
\begin{align*}
    x_s(\omega) = \frac{\mu(\omega) \pi(s | \omega)}{\pi(s)} ~ \text{ where } ~ \pi(s) = \sum_{\omega \in \Omega} \mu(\omega) \pi(s | \omega), \quad \forall \omega \in \Omega. 
\end{align*}
In the public persuasion problem, given an induced posterior $x_s$, each receiver $j \in [k]$ is willing to take action $+$ if and only if $\sum_{\omega \in \Omega} x_s(\omega) v_j(\omega) \ge 0$. 
Let $a_j^*(x_s) \in \{+, -\}$ denote the action taken by receiver $j \in [k]$ given posterior $x_s$:
\begin{align}\label{eq:receiver-public-action}
    a_j^*(x_s) = \begin{cases}
        + & \text{ if } \sum_{\omega \in \Omega} x_s(\omega) v_j(\omega) \ge 0 \\
        - & \text{ if } \sum_{\omega \in \Omega} x_s(\omega) v_j(\omega) < 0.
    \end{cases}
\end{align}
The sender's (expected) utility is the average utility obtained across all $k$ receivers:
\begin{align} \label{eq:public-sender-utility-definition}
    u^{\mathbf{Pub}}(\pi) = \sum_{s\in S} \pi(s) \frac{1}{k} \sum_{j=1}^k  \sum_{\omega \in \Omega} x_s(\omega) u_j( \omega, a_j^*(x_s)). 
\end{align}
The goal is to find a signaling scheme $\pi$ to maximize $u^{\mathbf{Pub}}(\pi)$. 
\begin{theorem}[\citet{dughmi_algorithmic_2017}]
\label{thm:public-NP-hard}
For any constant $c \in[0, \frac{1}{9}]$, it is NP-hard to solve, within additive approximation error $c$, 
public persuasion problems with $|\Omega| = k$ states and uniform prior $\mu(\omega) = \frac{1}{k}$, $\forall \omega \in [k]$.
\end{theorem}

We reduce the public persuasion problem to the best-response problem in multi-sender persuasion, which will prove that the latter problem is NP-hard. 
Given the public persuasion problem $\langle k, \Omega, \mu, \{ v_j(\omega), u_j(\omega) \}_{j\in [k], \omega \in \Omega} \rangle$, we construct the following best-response problem with two senders, where we fix sender $2$'s signaling scheme $\pi_2$ and find sender 1's best response: Let $C, N> 0$ and $M \ge \frac{N}{N-1}$ be some large numbers to be chosen later. 
\begin{itemize}
    \item There are $|\Omega| + k$ states, denoted by $\bar \Omega = \Omega \cup \{\bar \omega_1, \ldots, \bar \omega_j\}$, with prior $\bar \mu(\omega) = \frac{\mu(\omega)}{2}$ for $\omega \in \Omega$ and $\bar \mu(\bar \omega_j) = \frac{1}{2k}$ for $j=1, \ldots, k$. 
    \item The (single) receiver has $2k + 1$ actions, denoted by $ A = \{a_{1+}, a_{1-}, \ldots, a_{k+}, a_{k-}\} \cup \{ a_{\infty} \}$. 
    \item The receiver's utility is: 
    \begin{itemize}
        \item For any state $\omega \in \Omega$, let
        \begin{align*}
            & v(\omega, a_{j+}) = v_j(\omega) \\
            & v(\omega, a_{j-}) = 0 \\
            & v(\omega, a_{\infty}) = N.
        \end{align*}
        In words, for any state $\omega \in \Omega$, the receiver's two actions $a_{j+}, a_{j-}$ mimics receiver $j$'s actions $+, -$ in the public persuasion problem.  And $a_{\infty}$ is very attractive to the receiver.  
        \item For any state $\bar \omega_j$, $j\in [k]$, let
        \begin{align*}
            & v(\bar \omega_j, a_{\ell+}) = v(\bar \omega_j, a_{\ell-}) = -M \text{ for $\ell \ne j$} \\
            & v(\bar \omega_j, a_{j+}) = v(\bar \omega_j, a_{j-}) = 0 \\
            & v(\bar \omega_j, a_{\infty}) = -N. 
        \end{align*}
        In words, under state $\bar \omega_j$, the receiver is extremely unwilling to take actions other than $a_{j\pm}$.  And $a_{\infty}$ is very harmful to the receiver. 
    \end{itemize}
    \item Sender 1's utility $u_1(\cdot, \cdot)$ is: 
    \begin{itemize}
        \item For any state $\omega \in \Omega$,
        \begin{align*}
            & u_1(\omega, a_{j+}) = u_j(\omega, +) \quad \forall j\in [k]\\
            & u_1(\omega, a_{j-}) = u_j(\omega, -) \quad \forall j\in[k] \\
            & u_1(\omega, a_{\infty}) = - C. 
        \end{align*}
        In words, when the receiver takes actions $a_{j\pm}$, the sender obtains the same utility as if receiver $j$ takes action $\pm$ in the public persuasion problem.  But the sender suffers a large loss if the receiver takes $a_{\infty}$. 
        \item For any state $\bar \omega_j \in \Omega$, $j\in[k]$, 
        \begin{align*}
            & u_1(\bar \omega_j, a_{\ell+}) =  u_1(\bar \omega_j, a_{\ell-}) = 0 \quad \forall \ell\in[k] \\
            & u_1(\bar \omega_j, a_{\infty}) = - C. 
        \end{align*}
    \end{itemize}
    \item Sender 2's signaling scheme $\pi_2$ is the following: it sends $k$ possible signals $\{t_1, \ldots, t_k\}$ with probability: 
    \begin{align*}
        & \pi_2(t_j | \omega) = \frac{1}{k}, ~~ \forall j\in[k], ~ \forall \omega \in \Omega. \\
        & \pi_2(t_j | \bar \omega_j) = 1, ~~ \forall j\in[k]. 
    \end{align*}
\end{itemize}

We sketch both the public persuasion framework and the equivalent multi-sender construction outlined above in Fig \ref{fig:np_hard_1} and \ref{fig:np_hard_2}.

\subsubsection{Useful Claims Regarding Receiver's Behavior}
Before proving Theorem~\ref{thm:best-response-NP-hard}, we present some useful claims regarding the receiver's taking-best-action behavior.  First, we characterize the receiver's expected utilities when taking different actions in the multi-sender persuasion problem: 
\begin{claim}\label{claim:receiver-utility}
Let $x \in \Delta(\bar \Omega)$ be a distribution on the enlarged state space $\bar \Omega$.  Suppose sender $2$ sends signal $t_j$. Then, the receiver's expected utilities of taking different actions $a \in A$, denoted by $v(x, t_j, a)$, are: 
\begin{itemize}
    \item $v(x, t_j, a_{j+}) = \frac{1}{k} \sum_{\omega \in \Omega} x(\omega) v_j(\omega)$;
    \item $v(x, t_j, a_{j-}) = 0$; 
    \item $v(x, t_j, a_{\ell+}) = \frac{1}{k} \sum_{\omega \in \Omega} x(\omega) v_\ell(\omega) - x(\bar \omega_j) M~$ for $\ell \ne j$;
    \item $v(x, t_j, a_{\ell-}) = 0~$ for $\ell \ne j$;
    \item $v(x, t_j, a_{\infty}) = N\big(\frac{1}{k} \sum_{\omega \in \Omega} x(\omega) - x(\bar \omega_j)\big)$. 
\end{itemize}
\end{claim}
\begin{proof}
For any $a \in A$, by definition, 
\begin{align*}
    v(x, t_j, a) & = \sum_{\omega \in \bar \Omega} x(\omega) \pi_2(t_j | \omega) v(\omega, a) \\
    & = \sum_{\omega \in \Omega} x(\omega) \frac{1}{k} v(\omega, a) + \sum_{\ell=1}^k x(\bar \omega_\ell) \pi_2(t_j | \bar \omega_\ell) v(\bar \omega_\ell, a) \\
    & = \frac{1}{k} \sum_{\omega \in \Omega} x(\omega)  v(\omega, a) + x(\bar \omega_j) v(\bar \omega_j, a).
\end{align*}
Plugging in the definitions of utilities $v(\omega, a)$ and $v(\bar \omega_j, a)$ for different $a$ proves the claim. 
\end{proof}

As corollaries of the above claim, we have some guarantees when the receiver takes a best action: 
\begin{claim}\label{claim:not-a-infty-probability}
Given belief $x \in \Delta(\bar \Omega)$ and sender 2's signal $t_j$, if the receiver does not take $a_{\infty}$ as the best action, then it must be $\frac{1}{k} \sum_{\omega \in \Omega} x(\omega) \le \frac{N}{N-1} x(\bar \omega_j)$. 
\end{claim}
\begin{proof}
If $\frac{1}{k} \sum_{\omega \in \Omega} x(\omega) > \frac{N}{N-1} x(\bar \omega_j)$, then by Claim~\ref{claim:receiver-utility}, the receiver's utility of taking action $a_\infty$ is 
\begin{align*}
    v(x, t_j, a_\infty) = N\Big(\frac{1}{k} \sum_{\omega \in \Omega} x(\omega) - x(\bar \omega_j)\Big) > N\Big(\frac{1}{k} \sum_{\omega \in \Omega} x(\omega) - \frac{N-1}{N}\frac{1}{k} \sum_{\omega \in \Omega} x(\omega)\Big) = \frac{1}{k} \sum_{\omega \in \Omega} x(\omega) \ge v(x, t_j, a) 
\end{align*}
for any other actions $a \ne a_\infty$.  So, the receiver should take $a_\infty$, a contradiction. 
\end{proof}

\begin{claim}\label{claim:receiver-only-aj+-}
Given belief $x \in \Delta(\bar \Omega)$ and sender 2's signal $t_j$, if the receiver is unwilling to take $a_\infty$, then the receiver's utility of taking $a_{\ell\pm}$ for $\ell \ne j$ is $v(x, t_j, a_{\ell\pm}) \le 0$.  So, we can assume that the receiver will take $a_{j+}$ or $a_{j-}$. (Tie-breaking does not affect our conclusion.) 
\end{claim}
\begin{proof}
According to Claim~\ref{claim:not-a-infty-probability}, if the receiver is unwilling to take $a_\infty$, then $\frac{1}{k} \sum_{\omega \in \Omega} x(\omega) \le \frac{N}{N-1} x(\bar \omega_j)$.  This implies that the receiver's utility of taking $a_{\ell+}$ is, by Claim~\ref{claim:receiver-utility}
\begin{align*}
    v(x, t_j, a_{\ell+}) = \frac{1}{k} \sum_{\omega \in \Omega} x(\omega) v_\ell(\omega) - x(\bar \omega_j) M \le \frac{N}{N-1} x(\bar \omega_j) v_\ell(\omega) - x(\bar \omega_j) M \le x(\bar \omega_j) \Big( \frac{N}{N-1} - M \Big) \le 0, 
\end{align*}
under the assumption of $v_\ell(\omega) \le 1$ and $M \ge \frac{N}{N-1}$. 
\end{proof}

\begin{claim}\label{claim:same-action}
Let $x \in \Delta(\bar \Omega)$ be a belief on $\bar \Omega$.  And let $\tilde x \in \Delta(\Omega)$ be the conditional belief on $\Omega$: $\tilde x(\omega) = \frac{x(\omega)}{\sum_{\omega \in \Omega} x(\omega)}$, $\forall \omega \in \Omega$.  Fix any $j\in[k]$.  Suppose the receiver does not take action $a_\infty$ under signal $t_j$ in the multi-sender persuasion problem.  Then, the receiver takes action $a_{j+}$ (and $a_{j-}$) if and only if the receiver $j$ in the public persuasion problem takes action $+$ (and $-$) under belief $\tilde x$. 
\end{claim}
\begin{proof}
By Claim~\ref{claim:receiver-only-aj+-}, the receiver in the multi-sender problem must take action $a_{j+}$ or $a_{j-}$ if not taking $a_\infty$.  The receiver is willing to take $a_{j+}$ if and only if, by Claim~\ref{claim:receiver-utility},
\begin{align*}
    \frac{1}{k} \sum_{\omega \in \Omega} x(\omega) v_j(\omega) \ge 0 ~~ \iff ~~ \sum_{\omega \in \Omega} \frac{x(\omega)}{\sum_{\omega \in \Omega} x(\omega) } v_j(\omega) \ge 0 ~~ \iff ~~ \sum_{\omega \in \Omega} \tilde x(\omega) v_j(\omega) \ge 0, 
\end{align*}
which means that the receiver $j$ in the public persuasion problem is willing to take action $+$ under belief $\tilde x$ (see (\ref{eq:receiver-public-action})).
\end{proof}

\subsubsection{Proof of Theorem \ref{thm:best-response-NP-hard}}
Consider a signaling scheme $\pi_1 : \bar \Omega \to \Delta(S)$ of sender $1$, where $S$ is the signal space.  For a signal $s \in S$, let $x_s \in \Delta(\bar \Omega)$ be the posterior distribution over $\bar \Omega$ given $s$. And let $\pi_1(s) = \sum_{\omega \in \bar \Omega} \bar \mu(\omega) \pi_1(s|\omega)$ be the probability of sender 1 sending signal $s$.
A valid signaling scheme $\pi_1$ must satisfy the following Bayesian plausibility condition: 
\begin{equation} \label{eq:Bayes-plausibility-pi-1}
    \sum_{s \in S} \pi_1(s) x_s = \overline \mu ~~ \iff ~~ 
    \begin{cases}
        \sum_{s \in S} \pi_1(s)x_s(\omega)  = \overline \mu(\omega) = \frac{\mu(\omega)}{2} & \text{ for } \omega \in \Omega \\
        \sum_{s \in S} \pi_1(s)x_s(\bar \omega_j)  = \overline \mu(\bar \omega_j) = \frac{1}{2k} & \text{ for } j\in [k]
    \end{cases}. 
\end{equation}
Let $a^*(x_s, t_j)$ be the best action that the receiver will take when the posterior induced by sender 1 is $x_s$ (namely, sender 1 sends signal $s$) and sender 2 sends signal $t_j$.  According to Claim~\ref{claim:receiver-only-aj+-}, we have
\begin{equation} \label{eq:a^*-range}
    a^*(x_s, t_j) \in \{a_\infty, a_{j+}, a_{j-} \}. 
\end{equation}
Let $S_{\infty}$ be the set of signals of sender 1 for which the receiver will take action $a_\infty$ given some signal $t_j$ from sender 2: 
\begin{align*}
    S_{\infty} = \Big\{ s \in S ~\Big|~ a^*(x_s, t_j) = a_\infty \text{ for some } j\in [k] \Big\}. 
\end{align*}
Since $a_\infty$ is very harmful to sender 1 (causing utility $-C$), we show that the total probability of $S_{\infty}$ cannot be too large. 
\begin{lemma}\label{lem:S-infty-small}
If sender 1's expected utility under signaling scheme $\pi_1$ is $\ge 0$, then $\pi_1(S_{\infty}) = \sum_{s\in S_\infty} \pi_1(s) \le \frac{2k}{C} + \frac{1}{2N}$. 
\end{lemma}
\begin{proof}
Sender 1's expected utility is (fixing sender 2's scheme), 
\begin{align}
    u_1(\pi_1) & = \sum_{s\in S} \pi_1(s) \Big[ \sum_{\omega \in \bar \Omega} x_s(\omega) \sum_{j=1}^k \pi_2(t_j | \omega) u_1(\omega, a^*(x_s, t_j)) \Big]  \nonumber \\
    & = \sum_{s\in S} \pi_1(s) \Big[ \sum_{\omega \in \Omega} x_s(\omega) \sum_{j=1}^k \frac{1}{k} u_1(\omega, a^*(x_s, t_j)) + \sum_{j=1}^k x_s(\bar \omega_j) u_1(\bar \omega_j, a^*(x_s, t_j))  \Big]  \nonumber \\
    & = \sum_{s\in S} \pi_1(s) \sum_{j=1}^k \Big[ \frac{1}{k} \sum_{\omega \in \Omega} x_s(\omega) u_1(\omega, a^*(x_s, t_j)) + x_s(\bar \omega_j) u_1(\bar \omega_j, a^*(x_s, t_j))  \Big]  \label{eq:sender-1-utility} \\
    & = \sum_{s\in S_\infty} \pi_1(s) \sum_{j=1}^k \Big[ \frac{1}{k} \sum_{\omega \in \Omega} x_s(\omega) u_1(\omega, a^*(x_s, t_j)) + x_s(\bar \omega_j) u_1(\bar \omega_j, a^*(x_s, t_j))  \Big]  \nonumber \\
    & \quad\quad + \sum_{s\in S\setminus S_\infty} \pi_1(s) \sum_{j=1}^k \Big[ \frac{1}{k} \sum_{\omega \in \Omega} x_s(\omega) u_1(\omega, a^*(x_s, t_j)) + x_s(\bar \omega_j) u_1(\bar \omega_j, a^*(x_s, t_j))  \Big].  \nonumber 
\end{align}
Since the utility $u_1(\omega, a)$ is always $\le 1$, and when receiver takes action $a_\infty$ sender 1 gets utility $-C$,  
\begin{align*}
    u_1(\pi_1) & \le \sum_{s\in S_\infty} \pi_1(s) \sum_{j: a^*(x_s, t_j) = a_\infty} \Big[ \frac{1}{k} \sum_{\omega \in \Omega} x_s(\omega)  (-C) + x_s(\bar \omega_j) (-C)  \Big] \\
    & \quad \quad + \sum_{s\in S_\infty} \pi_1(s) \sum_{j=1}^k \Big[ \frac{1}{k} \sum_{\omega \in \Omega} x_s(\omega) \cdot 1 + x_s(\bar \omega_j) \cdot 1 \Big]  \\
    & \quad \quad + \sum_{s\in S\setminus S_\infty} \pi_1(s) \sum_{j=1}^k \Big[ \frac{1}{k} \sum_{\omega \in \Omega} x_s(\omega) \cdot 1 + x_s(\bar \omega_j) \cdot 1 \Big] \\ 
    & \le -C \sum_{s\in S_\infty} \pi_1(s) \Big[ \frac{1}{k} \sum_{\omega \in \Omega} x_s(\omega) + x_s(\bar \omega_j)  \Big] ~~ + ~~ \underbrace{\sum_{s\in S} \pi_1(s) \sum_{j=1}^k \Big[ \frac{1}{k} \sum_{\omega \in \Omega} x_s(\omega) + x_s(\bar \omega_j) \Big]}_{=1 \text{ by (\ref{eq:x-summation=1})}}. 
\end{align*}
Using $u_1(\pi_1) \ge 0$ and rearranging, we get $\sum_{s\in S_\infty} \pi_1(s) \Big[ \frac{1}{k} \sum_{\omega \in \Omega} x_s(\omega) + x_s(\bar \omega_j) \Big] \le \frac{1}{C}$, which implies
\begin{align*}
    \sum_{s\in S_\infty} \pi_1(s) \Big[ \frac{1}{k} \sum_{\omega \in \Omega} x_s(\omega) \Big] \le \frac{1}{C} ~~ \implies ~~ \sum_{s\in S_\infty} \pi_1(s) \sum_{\omega \in \Omega} x_s(\omega) \le \frac{k}{C}. 
\end{align*}
By the Bayesian plausibility condition (\ref{eq:Bayes-plausibility-pi-1}), we have
\begin{align*}
    \sum_{s\in S} \pi_1(s) \sum_{\omega \in \Omega} x_s(\omega) =
 \sum_{\omega \in \Omega} \sum_{s\in S} \pi_1(s) x_s(\omega) = \sum_{\omega \in \Omega} \overline{\mu}(\omega) = \frac{1}{2}. 
\end{align*}
So, 
\begin{align}
    \sum_{s\in S\setminus S_\infty} \pi_1(s) \sum_{\omega \in \Omega} x_s(\omega) = \frac{1}{2} - \sum_{s\in S_\infty} \pi_1(s) \sum_{\omega \in \Omega} x_s(\omega) \ge \frac{1}{2} - \frac{k}{C}.  \label{eq:two}
\end{align}
For any signal $s \in S\setminus S_\infty$, the receiver does not take $a_\infty$ under any $t_j$, which by Claim~\ref{claim:not-a-infty-probability} implies
\begin{align*}
    \frac{1}{k} \sum_{\omega \in \Omega} x_s(\omega) \le \frac{N}{N-1} x_s(\bar \omega_j) ~~ \implies ~~ x_s(\bar \omega_j) 
 \le \frac{N-1}{N} \frac{1}{k} \sum_{\omega \in \Omega} x_s(\omega)
\end{align*}
for all $j\in [k]$.  Moreover, because 
\begin{align} \label{eq:x-summation=1}
    \sum_{j=1}^k \Big[ \frac{1}{k} \sum_{\omega \in \Omega} x_s(\omega) + x_s(\bar \omega_j) \Big] = \sum_{\omega \in \bar \Omega} x(\omega) \sum_{j=1}^k \pi_2(t_j| \omega) = 1,
\end{align}
we have for any $s \in S\setminus S_\infty$, 
\begin{align} 
    1 \ge \sum_{j=1}^k \Big[ \frac{1}{k} \sum_{\omega \in \Omega} x_s(\omega) + \frac{N-1}{N} \frac{1}{k} \sum_{\omega \in \Omega} x_s(\omega) \Big] = \Big(2 - \frac{1}{N}\Big)\sum_{\omega \in \Omega} x_s(\omega)  ~~ \implies ~~ \sum_{\omega \in \Omega} x_s(\omega) \le \frac{1}{2 - \frac{1}{N}}.  \label{eq:three}
\end{align}
From (\ref{eq:two}) and (\ref{eq:three}) we get
\begin{align*}
    \frac{1}{2} - \frac{k}{C} \le \sum_{s\in S\setminus S_\infty} \pi_1(s) \frac{1}{2 - \frac{1}{N}} ~~ \implies ~~ \sum_{s\in S\setminus S_\infty} \pi_1(s) \ge 1 - \frac{2k}{C} - \frac{1}{2N},
\end{align*}
which proves the lemma since $\sum_{s\in S} \pi_1(s) = 1$. 
\end{proof}

We give another characterization of $\pi_1$: for most of the signals in $S\setminus S_\infty$, the total posterior probability for states in $\Omega$, $x_s(\Omega) = \sum_{\omega \in \Omega} x_s(\omega)$, should be close to $\frac{1}{2}$. Inequality (\ref{eq:three}) has shown an upper bound $\sum_{\omega \in \Omega} x_s(\omega) \le \frac{1}{2-\frac{1}{N}}$.  The following lemma gives a lower bound: 
\begin{lemma}\label{lem:S_<}
Fix any $\Delta > 0$.  Let
\begin{align*}
    S_{\ge} = \Big\{ s\in S\setminus S_\infty ~\Big|~ \frac{1}{2-\frac{1}{N}} \ge \sum_{\omega \in \Omega} x_s(\omega) \ge \frac{1}{2} - \Delta \Big\}, \quad \quad S_{<} = \Big\{ s\in S\setminus S_\infty ~\Big|~ \sum_{\omega \in \Omega} x_s(\omega) < \frac{1}{2} - \Delta \Big\}.
\end{align*}
(Note that $S_{\ge} \cup S_< = S\setminus S_\infty$ by (\ref{eq:three})).
We have $\pi_1(S_\ge)$ is large while $\pi_1(S_<)$ is small: 
\begin{align*}
    & \pi_1(S_\ge) = \sum_{s\in S_\ge} \pi_1(s) \ge 1 - \frac{2k}{C} - \frac{1}{2N} - \frac{1}{\Delta}\Big( \frac{1}{4N-2} + \frac{k}{C} \Big), \\
    & \pi_1(S_<) = \sum_{s\in S_<} \pi_1(s) \le \frac{1}{\Delta}\Big( \frac{1}{4N-2} + \frac{k}{C} \Big). 
\end{align*}
\end{lemma}
\begin{proof}
By (\ref{eq:two}), 
\begin{align*}
    \frac{1}{2} - \frac{k}{C} & \le \sum_{s\in S\setminus S_\infty} \pi_1(s) \sum_{\omega \in \Omega} x_s(\omega) ~ = ~ \sum_{s\in S_\ge} \pi_1(s) \sum_{\omega \in \Omega} x_s(\omega) + \sum_{s\in S_<} \pi_1(s) \sum_{\omega \in \Omega} x_s(\omega) \\
    & \le \frac{1}{2-\frac{1}{N}} \sum_{s\in S_\ge} \pi_1(s) + \big(\frac{1}{2} - \Delta\big)\sum_{s\in S_<} \pi_1(s) \\
    & \le -\Delta \sum_{s\in S_<} \pi_1(s) + \frac{1}{2-\frac{1}{N}} \Big(\sum_{s\in S_<} \pi_1(s) + \sum_{s\in S_\ge} \pi_1(s) \Big) \\
    & \le -\Delta \sum_{s\in S_<} \pi_1(s) + \frac{1}{2-\frac{1}{N}} \cdot 1
\end{align*}
So,
\begin{align*}
    \sum_{s\in S_<} \pi_1(s) \le \frac{1}{\Delta}\big( \frac{1}{2-\frac{1}{N}} - \frac{1}{2} + \frac{k}{C}\big) = \frac{1}{\Delta}\big( \frac{1}{4N-2} + \frac{k}{C} \big).
\end{align*}
Together with Lemma~\ref{lem:S-infty-small}, this implies $\pi_1(S_\ge) = 1 - \pi_1(S_\infty) - \pi_1(S_<) \ge 1 - \frac{2k}{C} - \frac{1}{2N} - \frac{1}{\Delta}\big( \frac{1}{4N-2} + \frac{k}{C} \big)$. 
\end{proof}

Now, we construct from $\pi_1$ a signaling scheme $\tilde \pi : \Omega \to \Delta(\tilde S)$ for the public persuasion problem.  The signal space of $\tilde \pi$ is $\tilde S = S_\ge \cup \{ s_0 \}$. For any $s \in S_{\ge}$, let the induced posterior $\tilde x_s \in \Delta(\Omega)$ be
\begin{align*}
    \tilde x_s(\omega) = \frac{x_s(\omega)}{\sum_{\omega \in \Omega} x_s(\omega)}
\end{align*}
(where $x_s$ is the posterior induced by $s$ in the signaling scheme $\pi_1$),
and denote
\begin{align*}
    \tilde \pi(s) = \frac{\pi_1(s)}{\sum_{s\in S_\ge} \pi_1(s)} \ge \pi_1(s),
\end{align*}
so $\sum_{s\in S_\ge} \tilde \pi(s) = 1$.
We will construct the posterior for $s_0$ later.
\begin{lemma} \label{lem:pi-tilde-prior-difference}
For any $\omega \in \Omega$, 
\begin{align*}
    \big| \sum_{s\in S_\ge} \tilde \pi(s) \tilde x_s(\omega) - \mu(\omega) \big| \le 4\Delta + \frac{2}{N} + \frac{4k}{C} + \frac{1}{\Delta}\Big(\frac{1}{2N-1} + \frac{2k}{C}\Big). 
\end{align*}
\end{lemma}
\begin{proof}
On the one hand,
\begin{align*}
    \sum_{s\in S_\ge} \tilde \pi(s) \tilde x_s(\omega) & \ge \sum_{s\in S_\ge} \pi_1(s) \frac{x_s(\omega)}{\sum_{\omega \in \Omega} x_s(\omega)} \ge \sum_{s\in S_\ge} \pi_1(s) \frac{x_s(\omega)}{\frac{1}{2-\frac{1}{N}}} = \Big(2-\frac{1}{N}\Big) \sum_{s\in S_\ge} \pi_1(s) x_s(\omega) \\
    \text{by (\ref{eq:Bayes-plausibility-pi-1})} & = \Big(2-\frac{1}{N}\Big) \Big( \frac{\mu(\omega)}{2} - \sum_{s\in S_\infty \cup S_<} \pi_1(s) x_s(\omega) \Big) \\
    & \ge \Big(2-\frac{1}{N}\Big) \Big( \frac{\mu(\omega)}{2} - \sum_{s\in S_\infty} \pi_1(s) -  \sum_{s\in S_<} \pi_1(s) \Big) \\
    \text{by Lemma~\ref{lem:S-infty-small} and \ref{lem:S_<}} & \ge \Big(2-\frac{1}{N}\Big) \Big( \frac{\mu(\omega)}{2} - \frac{2k}{C} - \frac{1}{2N} - \frac{1}{\Delta}\big( \frac{1}{4N-2} + \frac{k}{C} \big) \Big) \\
    & \ge \mu(\omega)  - \frac{\mu(\omega)}{2N} - \frac{4k}{C} - \frac{1}{N} - \frac{1}{\Delta}\big( \frac{1}{2N-1} + \frac{2k}{C} \big) \\
    & \ge \mu(\omega)  - \frac{2}{N} - \frac{4k}{C} - \frac{1}{\Delta}\big( \frac{1}{2N-1} + \frac{2k}{C} \big). 
\end{align*}
On the other hand, 
\begin{align*}
    \sum_{s\in S_\ge} \tilde \pi(s) \tilde x_s(\omega) & = \sum_{s\in S_\ge} \frac{\pi_1(s)}{\sum_{s\in S_\ge} \pi_1(s)} \frac{x_s(\omega)}{\sum_{\omega \in \Omega} x_s(\omega)} \\
    \text{(by definition of $S_\ge$)} ~ & \le \sum_{s\in S_\ge} \frac{\pi_1(s)}{\sum_{s\in S_\ge} \pi_1(s)} \frac{x_s(\omega)}{\frac{1}{2} - \Delta} \\
    & = \frac{2}{1-2\Delta}\frac{1}{\sum_{s\in S_\ge} \pi_1(s)}\sum_{s\in S_\ge} \pi_1(s) x_s(\omega) \\
    & \le \frac{2}{1-2\Delta}\frac{1}{\sum_{s\in S_\ge} \pi_1(s)}\sum_{s\in S} \pi_1(s) x_s(\omega) \\
    \text{by (\ref{eq:Bayes-plausibility-pi-1})} ~ & = \frac{2}{1-2\Delta}\frac{1}{\sum_{s\in S_\ge} \pi_1(s)} \frac{\mu(\omega)}{2} \\
    \text{by Lemma \ref{lem:S_<}} ~ & \le \frac{2}{1-2\Delta}\frac{1}{1 - \frac{2k}{C} - \frac{1}{2N} - \frac{1}{\Delta}\big( \frac{1}{4N-2} + \frac{k}{C} \big)} \frac{\mu(\omega)}{2} \\
    & \le \Big( 1 + 4\Delta + \frac{4k}{C} + \frac{1}{N} + \frac{1}{\Delta}\big( \frac{1}{2N-1} + \frac{2k}{C} \big)\Big) \mu(\omega) \\
    & \le \mu(\omega) + 4\Delta + \frac{4k}{C} + \frac{1}{N} + \frac{1}{\Delta}\big( \frac{1}{2N-1} + \frac{2k}{C} \big).
\end{align*}
Two above two cases together prove the lemma. 
\end{proof}

As shown in Lemma~\ref{lem:pi-tilde-prior-difference}, the signaling scheme $\tilde \pi$ with signals in $S_\ge$ may not satisfy the Bayesian plausibility condition $\sum_{s \ge S_\ge} \tilde \pi(s) \tilde x_s = \mu(\omega)$. 
That is why we need the additional signal $s_0$.  We want to find a posterior $y \in \Delta(\Omega)$ for signal $s_0$, and a coefficient $\alpha \in [0, 1]$ such that the following convex combination of $\{\tilde x_s\}_{s \in S_\ge}$ and $y$ satisfies Bayesian plausibility:
\begin{align}\label{eq:convex-combination}
    (1-\alpha) \sum_{s\in S_\ge} \tilde \pi(s) \tilde x_s + \alpha y = \mu. 
\end{align}
\begin{lemma}\label{lem:alpha-small}
Suppose $\min_{\omega \in \Omega} \mu(\omega) \ge p_0 \ge 2\big( 4\Delta + \frac{2}{N} + \frac{4k}{C} + \frac{1}{\Delta}(\frac{1}{2N-1} + \frac{2k}{C} ) \big) > 0$.  Then, there exists $y \in \Delta(\Omega)$ and $\alpha \le \frac{2}{p_0}\big( 4\Delta + \frac{2}{N} + \frac{4k}{C} + \frac{1}{\Delta}(\frac{1}{2N-1} + \frac{2k}{C} ) \big)$ that satisfy (\ref{eq:convex-combination}). 
\end{lemma}
\begin{proof}
Let $z = \sum_{s\in S_\ge} \tilde \pi(s) \tilde x_s$.  By Lemma~\ref{lem:pi-tilde-prior-difference}, we have 
\begin{align*}
    \| z - \mu \|_\infty = \max_{\omega \in \Omega} |z(\omega) - \mu(\omega)| \le 4\Delta + \frac{2}{N} + \frac{4k}{C} + \frac{1}{\Delta}\Big(\frac{1}{2N-1} + \frac{2k}{C}\Big). 
\end{align*}
To satisfy (\ref{eq:convex-combination}), which is equivalent to
\begin{align*}
    (1-\alpha) z + \alpha y = \mu ~~ \iff ~~ \alpha (y-z) = \mu - z, 
\end{align*}
we can let $y$ be the intersection of the ray starting from $z$ pointing towards $\mu$ and the boundary of $\Delta(\Omega)$.  By doing this, $y-\mu$ and $z - \mu$ are in the same direction and 
\begin{align*}
    \alpha = \frac{\| \mu-z \|_\infty}{\| y-z \|_\infty}.
\end{align*}
Since $y$ is on the boundary of $\Delta(\Omega)$, some $y(\omega)$ must be $0$.  So,
\begin{align*}
    \| y-z \|_\infty \ge \min_{\omega \in \Omega} z(\omega) & \ge \min_{\omega \in \Omega} \mu(\omega) - \Big( 4\Delta + \frac{2}{N} + \frac{4k}{C} + \frac{1}{\Delta}\big(\frac{1}{2N-1} + \frac{2k}{C}\big) \Big) \\
    & \ge p_0 - \frac{p_0}{2} = \frac{p_0}{2}. 
\end{align*}
This implies
\begin{align*}
    \alpha \le \frac{2}{p_0} \Big( 4\Delta + \frac{2}{N} + \frac{4k}{C} + \frac{1}{\Delta}\big(\frac{1}{2N-1} + \frac{2k}{C}\big) \Big).
\end{align*}
\end{proof}

With (\ref{eq:convex-combination}) satisfied, $\tilde \pi$ now is a valid signaling scheme for the public persuasion problem, which sends signal $s \in S_\ge$ with probability $(1-\alpha)\tilde \pi(s)$, inducing posterior $\tilde x_s$, and sends signal $s_0$ with probability $\alpha$, inducing posterior $y$.  Let's consider the sender's utility (\ref{eq:public-sender-utility-definition}) in the public persuasion problem using $\tilde \pi$: 
\begin{align*}
u^{\mathbf{Pub}}(\tilde \pi) & = (1-\alpha) \sum_{s\in S_\ge} \tilde \pi(s) \frac{1}{k} \sum_{j=1}^k \sum_{\omega \in \Omega} \tilde x_s(\omega) u_j(\omega, a^*_j(\tilde x_s))  + \alpha \frac{1}{k} \sum_{j=1}^k \sum_{\omega \in \Omega} y(\omega) u_j(\omega, a^*_j(y)) \\
& \ge \sum_{s\in S_\ge} \tilde \pi(s) \frac{1}{k} \sum_{j=1}^k \sum_{\omega \in \Omega} \tilde x_s(\omega) u_j(\omega, a^*_j(\tilde x_s))  - \alpha \quad\quad  \text{because $0 \le u_j(\omega, a) \le 1$}. 
\end{align*}
By Claim~\ref{claim:same-action}, the receiver $j$'s best action $a^*_j(\tilde x_s)$ is $+$ (and $-$) if and only if the receiver in the multi-sender problem takes action $a^*(x_s, t_j) = a_{j+}$ (and $a_{j-}$) given posterior $x_s$ from sender 1 and signal $t_j$ from sender 2.  So, by the definition of sender 1's utility in the multi-sender problem,
\begin{align*}
    u_j(\omega, a^*_j(\tilde x_s)) = u_1(\omega, a^*(x_s, t_j)).
\end{align*}
Then, we have
\begin{align*}
u^{\mathbf{Pub}}(\tilde \pi) & \ge \sum_{s\in S_\ge} \tilde \pi(s) \frac{1}{k} \sum_{j=1}^k \sum_{\omega \in \Omega} \tilde x_s(\omega) u_1(\omega, a^*(x_s, t_j)) - \alpha \\
& \ge \sum_{s\in S_\ge} \pi_1(s) \frac{1}{k} \sum_{j=1}^k \sum_{\omega \in \Omega} \frac{x_s(\omega)}{\sum_{\omega \in \Omega} x_s(\omega) } u_1(\omega, a^*(x_s, t_j)) - \alpha \\
& \ge \Big( 2 - \frac{1}{N} \Big) \sum_{s\in S_\ge} \pi_1(s) \frac{1}{k} \sum_{j=1}^k \sum_{\omega \in \Omega} x_s(\omega) u_1(\omega, a^*(x_s, t_j)) - \alpha.
\end{align*}
On the other hand, let's consider sender 1's utility in the multi-sender problem with signaling scheme $\pi_1$.  By Equation (\ref{eq:sender-1-utility}),  
\begin{align*}
    u_1(\pi_1)
    & = \sum_{s\in S} \pi_1(s) \sum_{j=1}^k \Big[ \frac{1}{k} \sum_{\omega \in \Omega} x_s(\omega) u_1(\omega, a^*(x_s, t_j)) + x_s(\bar \omega_j) u_1(\bar \omega_j, a^*(x_s, t_j))  \Big]  \\
    & \le \sum_{s\in S_\infty \cup S_<} \pi_1(s) \cdot 1 \hspace{5em} \text{because $u_1(\cdot, \cdot) \le 1$} \\
    & \quad ~~ + \sum_{s\in S_\ge} \pi_1(s) \sum_{j=1}^k \Big[ \frac{1}{k} \sum_{\omega \in \Omega} x_s(\omega) u_1(\omega, a^*(x_s, t_j)) + x_s(\bar \omega_j) \underbrace{u_1(\bar \omega_j, a^*(x_s, t_j))}_{=0 \text{ because $a^*(x_s, t_j) \in \{a_{j+}, a_{j-}\}$ from (\ref{eq:a^*-range})}}  \Big]  \\
    & \le \frac{2k}{C} + \frac{1}{2N} + \frac{1}{\Delta}\Big(\frac{1}{4N-2} + \frac{k}{C} \Big)  \hspace{5em} \text{by Lemma \ref{lem:S-infty-small} and \ref{lem:S_<}} \\
    & \quad ~~ + \sum_{s\in S_\ge} \pi_1(s) \sum_{j=1}^k \Big[ \frac{1}{k} \sum_{\omega \in \Omega} x_s(\omega) u_1(\omega, a^*(x_s, t_j)) \Big] 
\end{align*}
So, $\sum_{s\in S_\ge} \pi_1(s) \sum_{j=1}^k \big[ \frac{1}{k} \sum_{\omega \in \Omega} x_s(\omega) u_1(\omega, a^*(x_s, t_j)) \big]  \ge u_1(\pi_1) - \frac{2k}{C} - \frac{1}{2N} - \frac{1}{\Delta}\big(\frac{1}{4N-2} + \frac{k}{C} \big)$.  This implies 
\begin{align}\label{eq:u^pub-lower-bound}
    u^{\mathbf{Pub}}(\tilde \pi) & \ge \Big( 2 - \frac{1}{N} \Big) \Big[ u_1(\pi_1) - \frac{2k}{C} - \frac{1}{2N} - \frac{1}{\Delta}\big(\frac{1}{4N-2} + \frac{k}{C} \big) \Big] - \alpha.
\end{align}

Finally, we prove that if the signaling scheme $\pi_1$ is nearly optimal for the multi-sender best-response problem, then the corresponding scheme $\tilde \pi$ for the public persuasion problem must be nearly optimal as well.
\begin{claim}\label{claim:approximately-optimal}
If $\pi_1$ is approximately optimal for sender 1's best-response problem up to additive error $c$, then the $\tilde \pi$ constructed above is approximately optimal for the public persuasion problem with additive error $2c + \frac{4k}{C} + \frac{2}{N} + \frac{1}{\Delta}\big(\frac{1}{2N-1} + \frac{2k}{C} \big) + \alpha$.
\end{claim}
\begin{proof}
Let $\pi^*$ be the optimal signaling scheme for the public persuasion problem, which induces posterior $x^*_s \in \Delta(\Omega)$ at signal $s \in S^*$.  Let $\pi'_1$ be the following signaling scheme for sender 1 in the multi-sender problem: for any signal $s \in S^*$, the probability of the signal is $\pi'_1(s) = \pi^*(s)$ and the induced posterior $x'_s \in \Delta(\bar \Omega)$ is  
\begin{align*}
    x'_s(\omega) = \frac{x^*_s(\omega)}{2}, \quad x'_s(\bar \omega_j) = \frac{1}{2k}. 
\end{align*}
It is easy to verify that $\pi'_1$ is valid (satisfying Bayesian plausibility (\ref{eq:Bayes-plausibility-pi-1})).  We then note that, at each posterior $x'_s$, the receiver's utility of taking action $a_\infty$ is always $0$ regardless of sender $2$'s signal $t_j$:
\begin{align*}
    v(x_s', t_j, a_{\infty}) = N\Big(\frac{1}{k} \sum_{\omega \in \Omega} x_s'(\omega) - x_s'(\bar \omega_j)\Big) = N\Big(\frac{1}{k} \frac{1}{2} - \frac{1}{2k}\Big) = 0.
\end{align*}
So, we can assume that the receiver will take $a_{j+}$ or $a_{j-}$ by Claim~\ref{claim:receiver-only-aj+-}.  Moreover, by Claim~\ref{claim:same-action}, the receiver takes $a_{j+}$ and $a_{j-}$ if and only if the receiver $j$ in the public persuasion problem with belief $x^*_s$ takes action $+$ and $-$.  So, 
\begin{align*}
    u_1(\omega, a^*(x_s', t_j)) = u_j(\omega, a^*_j(x_s^*)). 
\end{align*}
This means that the utility of sender 1 in the multi-sender problem satisfies: 
\begin{align*}
    u_1(\pi_1') & = \sum_{s^*\in S} \pi_1'(s) \sum_{j=1}^k \Big[ \frac{1}{k} \sum_{\omega \in \Omega} x_s'(\omega) u_1(\omega, a^*(x_s, t_j)) + x_s'(\bar \omega_j) \underbrace{u_1(\bar \omega_j, a^*(x_s', t_j))}_{=0 \text{ because $a^*(x_s', t_j) \ne a_\infty$}}  \Big] \\
    & = \sum_{s^*\in S} \pi^*(s) \sum_{j=1}^k \Big[ \frac{1}{k} \sum_{\omega \in \Omega} \frac{x_s^*(\omega)}{2} u_j(\omega, a^*_j(x_s^*))  \Big] \\
    & = \frac{1}{2} \sum_{s^*\in S} \pi^*(s) \frac{1}{k} \sum_{j=1}^k \sum_{\omega \in \Omega} x_s^*(\omega) u_j(\omega, a^*_j(x_s^*)) ~ = ~ \frac{1}{2} u^{\mathbf{Pub}}(\pi^*). 
\end{align*}
If $\pi_1$ is approximately optimal up to additive error $c$ in the multi-sender best-response problem, then
\begin{align*}
    u_1(\pi_1) \ge u_1(\pi_1') - c
\end{align*}
Plugging this into (\ref{eq:u^pub-lower-bound}), 
\begin{align*}
    u^{\mathbf{Pub}}(\tilde \pi) & \ge \Big( 2 - \frac{1}{N} \Big) \Big[ u_1(\pi_1') - c - \frac{2k}{C} - \frac{1}{2N} - \frac{1}{\Delta}\big(\frac{1}{4N-2} + \frac{k}{C} \big) \Big] - \alpha \\
    & = \Big( 2 - \frac{1}{N} \Big) \Big[ \frac{1}{2} u^{\mathbf{Pub}}(\pi^*) - c - \frac{2k}{C} - \frac{1}{2N} - \frac{1}{\Delta}\big(\frac{1}{4N-2} + \frac{k}{C} \big) \Big] - \alpha \\
    & \ge u^{\mathbf{Pub}}(\pi^*) - \frac{u^{\mathbf{Pub}}(\pi^*) }{2N} - 2c - \frac{4k}{C} - \frac{1}{N} - \frac{1}{\Delta}\big(\frac{1}{2N-1} + \frac{2k}{C} \big) \Big] - \alpha \\
    & \ge u^{\mathbf{Pub}}(\pi^*) - 2c - \frac{4k}{C} - \frac{2}{N} - \frac{1}{\Delta}\big(\frac{1}{2N-1} + \frac{2k}{C} \big) - \alpha.
\end{align*}
This means that $\tilde \pi$ is approximately optimal for the public persuasion problem up to additive error $2c + \frac{4k}{C} + \frac{2}{N} + \frac{1}{\Delta}\big(\frac{1}{2N-1} + \frac{2k}{C} \big) + \alpha$. 
\end{proof}

\textbf{We now prove Theorem~\ref{thm:best-response-NP-hard}.}
Let $\langle k, \Omega, \mu, \{ v_j(\omega), u_j(\omega) \}_{j\in [k], \omega \in \Omega} \rangle$ be any public persuasion problem with $|\Omega| = k$ states and uniform prior $\mu(\omega) = \frac{1}{k} = p_0$.  Construct the multi-sender best-response problem as above (where the range of utility of sender 1 is $[-C, 1]$). 
If we can find an $\eps$-approximately optimal signaling scheme $\pi_1$ for sender 1's best-response problem with utility range $[-1, 1]$, with
\begin{align*}
    \eps = \frac{1}{k^6},
\end{align*}
then $\pi_1$ is a $C\eps$-approximately optimal signaling scheme with utility range $[-C, 1]$.  Then by Claim~\ref{claim:approximately-optimal}, the scheme $\tilde \pi$ constructed above is approximately optimal for the public persuasion problem with additive error at most
\begin{align*}
    2 C\eps & + \frac{4k}{C} + \frac{2}{N} + \frac{1}{\Delta}\big(\frac{1}{2N-1} + \frac{2k}{C} \big) + \alpha \\
    \text{by Lemma \ref{lem:alpha-small}} ~ & \le 2 C\eps + \frac{4k}{C} + \frac{2}{N} + \frac{1}{\Delta}\big(\frac{1}{2N-1} + \frac{2k}{C} \big) + \frac{2}{p_0} \Big( 4\Delta + \frac{2}{N} + \frac{4k}{C} + \frac{1}{\Delta}\big(\frac{1}{2N-1} + \frac{2k}{C}\big) \Big) \\
    & \le 2 C\eps + \big( 2k + 1 \big) \Big( 4\Delta + \frac{2}{N} + \frac{4k}{C} + \frac{1}{\Delta}\big(\frac{1}{2N-1} + \frac{2k}{C}\big) \Big).
\end{align*}
Let $C = k^5, N = k^4, \Delta = \frac{1}{k^2}$.
\begin{align*}
    & \le 2 k^5 \eps + \big( 2k + 1 \big) \Big( \frac{4}{k^2} + \frac{2}{k^4} + \frac{4k}{k^5} + k^2\big(\frac{1}{2k^4-1} + \frac{2k}{k^5}\big) \Big) = O\big( \frac{1}{k} \big) \le \frac{1}{9},
\end{align*}
for sufficiently large $k$.  Theorem~\ref{thm:public-NP-hard} says that finding $\frac{1}{9}$-approximation for the public persuasion is NP-hard.  So, finding $\eps = \frac{1}{k^6}$-approximation for the multi-sender best-response problem is NP-hard. 

\newpage
\subsection{Proof of Theorem~\ref{thm:full-revelation}}\label{app:full_reveal_proof}
\begin{proof}
    Since at each state $\omega$, there is a unique optimal action $a$ it suffices to consider $|\A| \leq |\Omega|$. Next, let the signal space be $|S| = |\A|^{\tfrac{1}{n-1}} \triangleq k$; we shall see this is without loss of generality when the signal space is larger. We first give a construction for a mapping $\alpha$ between all $k$-ary strings of length $n$ (all possible joint signals) to $|\A|$. Let $\zeta$ denote a subset of these strings such that for any two strings $\bm{s}^1 \in \zeta$,  $\bm{s}^2 \in \zeta$ the hamming distance between them is at least two - $d_H(\bm{s}^1, \bm{s}^2) \geq 2$. The $k$-ary Gray code is an ordering of all unique $k$-ary strings of length $n$ such that any two consecutive strings are exactly 1 apart in hamming distance. Such a construction is always possible \citep{guan1998generalized}. Since there are $k$ different values possible at any position, within at least every $k$ strings in the grey code, we should have two strings that are hamming distance 2 apart. Thus $|\zeta| \geq k^{n-1} = |\A|$. This is indeed tight since $k^{n-1}$ is the total number of unique $n-1$ length $k$-ary strings possible - thus if $|\zeta| > k^{n-1}$, it would mean there are two strings where that match on $n-1$ positions, violating the construction of $\zeta$. We construct $\alpha$ as follows: map each string in $\zeta$ to a unique action in $\A$ and assign the remaining joint signal strings arbitrarily to an action. 

    Under this mapping, we now give a constructive joint signaling scheme that is (1) a Pure Nash Equilibrium and (2) fully reveals the optimal action to the agent. Let $\alpha^{-1}(a)$ map to the joint signal $\bm{s} \in \zeta$ such that $\alpha(\bm{s}) = a$. Further, let $f: \Omega \rightarrow \A$ denote the unique agent-optimal action under state $\omega$, with its inverse $f^{-1}(a)$ denoting the set of states for which this action is agent-optimal. Next, consider the following joint signaling scheme: for all $\bm{s} \in \zeta$, $\pi(\bm{s}|\omega) = 1$ if $\omega \in f^{-1}(\alpha(\bm{s}))$. That is for any $\omega$, the joint signal $\bm{s} \in \zeta$ that corresponds to the optimal agent action under $\omega$, i.e. $\alpha(\bm{s}) = f(\omega)$, is sent with probability 1. The agent can thus uniquely map each joint signal realization to a set of states wherein a fixed action is optimal. In other words, this fully reveals the optimal action for the agent at any state realization $\omega$. To show this is a Nash equilibrium, observe that since all strings in $\zeta$ are hamming distance at least 2 apart, there is in fact a bijection between any $n-1$ sub-signal/sub-string within $\zeta$ and the action. Thus, each optimal action is fully specified by signals of just $n-1$ agents. So if a sender unilaterally shifts her signaling, the agent can observe that $n-1$ signals still uniquely map to states that share a common optimal action, and essentially ignore the deviating agent's signal. Thus, no change in agent belief or action occurs, leading the deviation to be non-beneficial. Since the choice of deviating agent here is arbitrary, this presented scheme is a pure Nash equilibrium. 

    However, full revelation equilibrium is not unique, which we show through an example. Consider $n=2$ senders, with $|\mathcal{A}| = 4$ actions, and $|\Omega| = 4$ states, with the following prior: $[0.15, 0.35, 0.15, 0.35]$. Sender $1$ has utility 1 whenever action 1 is taken and 0 otherwise. Similarly, sender $2$ has utility 1 whenever action 3 is taken and 0 otherwise. Note both utilities are agnostic to the state $\omega$. The receiver utility is given by the following matrix:

    \begin{equation}
        V = \begin{bmatrix}
                1 & -1 & 0 & 0\\
                -1 & 1 & 0 & 0\\
                0 & 0 & 1 & -1\\
                0 & 0 & -1 & 1\\
            \end{bmatrix}
    \end{equation}

    Under a full-revelation or optimal action revelation equilibrium, note that each sender would get utility $0.15$. Now consider the following signaling scheme using only 3 signals, which we express as a $|\Omega| \times |\Sig|$ matrix\footnote{a scheme with 3 signals can without loss of generality be extended to a scheme with 4 signals, which is what the optimal receiver action revelation scheme uses.}. 
    
    \begin{equation}
            \pi_1 = \begin{bmatrix}
                0  &  1  &  0      \\
                \tfrac{4}{7} & \tfrac{3}{7} & 0 \\
                1 & 0 & 0 \\
                1 & 0 & 0 \\
            \end{bmatrix}
            \quad \pi_2 = 
            \begin{bmatrix}
                0  &  0 & 1      \\
                0 & 0 & 1 \\
                0 & 1 & 0 \\
                \tfrac{4}{7} & \tfrac{3}{7} & 0
            \end{bmatrix} 
    \end{equation}

    Joint signal realizations $01$ and $12$ from such a scheme induces the following posterior beliefs with probability $0.3$:
    \begin{gather}
        \mu_{01} = [0, 0, 0.5, 0.5] \quad \mu_{12} = [0.5, 0.5, 0, 0]
    \end{gather}
    Note that for any tie-breaking rule that favors senders, the posteriors above give utility 0.3 to both senders. All other posteriors have dominant actions that give 0 utilities to both senders. We can use the optimization program presented in proposition \ref{prop:best-response-bi-linear} to verify this is an equilibrium. 
\end{proof}

\newpage

\subsection{Proof of Theorem~\ref{thm:PPAD-non-fixed}}
\label{app:PPAD-non-fixed}
\begin{proof}
It is known that finding Nash equilibria in 2-player games with 0/1 utilities is PPAD-hard \cite{abbott_complexity_2005, chen_settling_2009}.  We reduce this PPAD-hard problem to multi-sender persuasion, which proves that the latter problem is also PPAD-hard.  Let $\hat u_1, \hat u_2 \in \{0, 1\}^{m\times m}$ be the utility matrices of the $2$ players, where $m$ is the number of actions of each player.  We construct a multi-sender persuasion game as follows: 
\begin{itemize}
    \item There are 2 states $\Omega=\{0, 1\}$ with prior $\mu_0(0) = \mu_0(1) = 1/2$, $|\A| = 4$ actions for the receiver labeled as $\A = \{a_{00}, a_{01}, a_{10}, a_{11}\}$, and $n=2$ senders each having a signal space $\Sig = \{1, \ldots, m\}$. 
    \item The receiver's utility is $0$ regardless of actions and states, so he is indifferent among taking all actions. Suppose the receiver breaks ties in the following way: given joint signal $(s_1, s_2)$ from the 2 senders, take action
    \begin{equation*}
        \alpha(s_1, s_2) = \begin{cases}
            a_{00} & \text{if } ~ \hat u_1(s_1, s_2) = 0, \, \hat u_2(s_1, s_2)=0; \\
            a_{01} & \text{if } ~ \hat u_1(s_1, s_2) = 0, \, \hat u_2(s_1, s_2)=1; \\
            a_{10} & \text{if } ~ \hat u_1(s_1, s_2) = 1, \, \hat u_2(s_1, s_2)=0; \\
            a_{11} & \text{if } ~ \hat u_1(s_1, s_2) = 1, \, \hat u_2(s_1, s_2)=1. 
        \end{cases}
    \end{equation*}
    \item The utility of each sender $i$ is: 
    \begin{align*}
        & u_i(a, \omega=1) =
        \begin{cases}
            \hat u_i(s_1, s_2) & \text{if there exist $s_1, s_2\in\Sig$ such that $\alpha(s_1, s_2) = a$;} \\
            0 & \text{otherwise.}
        \end{cases} \\
        & u_i(a, \omega=0) = 0, ~~~~ \forall a\in\A.
    \end{align*}
    We note that the first equation is well-defined, because for any different joint signals $(s_1, s_2)$ and $(s_1', s_2')$, if they both satisfy $\alpha(s_1, s_2) = \alpha(s_1', s_2') = a$, then they define the same utility $u_i(a, \omega=1) = \hat u_i(s_1, s_2) = \hat u_i(s_1', s_2')$. 
\end{itemize}
We note that the expected utility of each sender $i$ under signaling schemes $\bpi = (\pi_1, \pi_2)$ is equal to
\begin{align*}
 \ubar_i(\bpi) & = \sum_{\omega\in\Omega} \sum_{\bs\in \Sig^n} \mu_0(\omega) \bpi(\bs | \omega) u_i(\alpha(\bs), \omega) \\
 & = \frac{1}{2}\cdot 0 + \frac{1}{2} \sum_{s_1, s_2} \pi_1(s_1 | \omega=1) \pi_2(s_2 | \omega=1) u_i(\alpha(s_1, s_2), \omega = 1) \\
 & = \frac{1}{2} \sum_{s_1, s_2} \pi_1(s_1 | \omega=1) \pi_2(s_2 | \omega=1) \hat u_i(s_1, s_2) \\
 & = \frac{1}{2} \hat u_i(x_1, x_2), 
\end{align*}
where $\hat u_i(x_1, x_2)$ is the expected utility of player $i$ in the 2-player 0/1 utility game when the two players use mixed strategies $x_1, x_2$ where player $i$ samples action $s_i \in \{1, \ldots, m\}$ with probability $x_i(s_i) = \pi_i(s_i | \omega=1)$.  If we can find an NE $(\pi_1, \pi_2)$ for the multi-sender persuasion game, then the corresponding mixed strategy profile $(x_1, x_2)$ where $x_i(s_i) = \pi_i(s_i | \omega=1)$ is an NE for the 2-player 0/1 utility game, which is PPAD-hard to find. 
\end{proof}

\section{Find Local NE via Deep Learning}\label{appx:deep}
In this section, we describe the settings of our deep learning experiments and show more results.

For each problem instance, we collect a dataset comprising 50,000 randomly selected samples and train the networks for 30 epochs using the Adam optimizer~\cite{kingma2014adam} with a learning rate of 0.01. For extra-gradient, we initiate the optimization process from a set of 300 random starting points. For each starting point, we run 20 iterations of extra-gradient updates with the Adam optimizer and a learning rate of 0.1. We then use the result with the highest social welfare to compare the performance of different algorithms.

To evaluate if a joint signaling policy profile $\bpi$ derived from the extra-gradient algorithm constitutes a local NE, we randomly select $K$ policies $\bpi'_j$ for each sender $j$ within the vicinity $\{\pi'_j\ |\ \|\pi'_j-\pi_{\phi_j}\|_\infty\le \epsilon\}$. We then verify if any of these deviations result in increased utility. The number of test samples $K$ grows linearly with the problem size:
\begin{align}
    K = \min\left\{10000, 1000 * (n-1)(|\Omega|-1)(|\mathcal S|-1)(|\mathcal A|-1)\right\}.
\end{align}

In the main text, we set the neighborhood size $\epsilon$ to 0.005. Now we apply a more stringent criterion for $\epsilon$-local NE, with $\epsilon$ set to 0.01 and the extra-gradient optimization step increased to 30 accordingly. We reassess our method under this setting against baselines in Fig.~\ref{fig:2_sender_ln}. As we can observe, the performance of our method is still significantly better than other algorithms.

\begin{figure*}
    \centering
    \includegraphics[width=\linewidth]{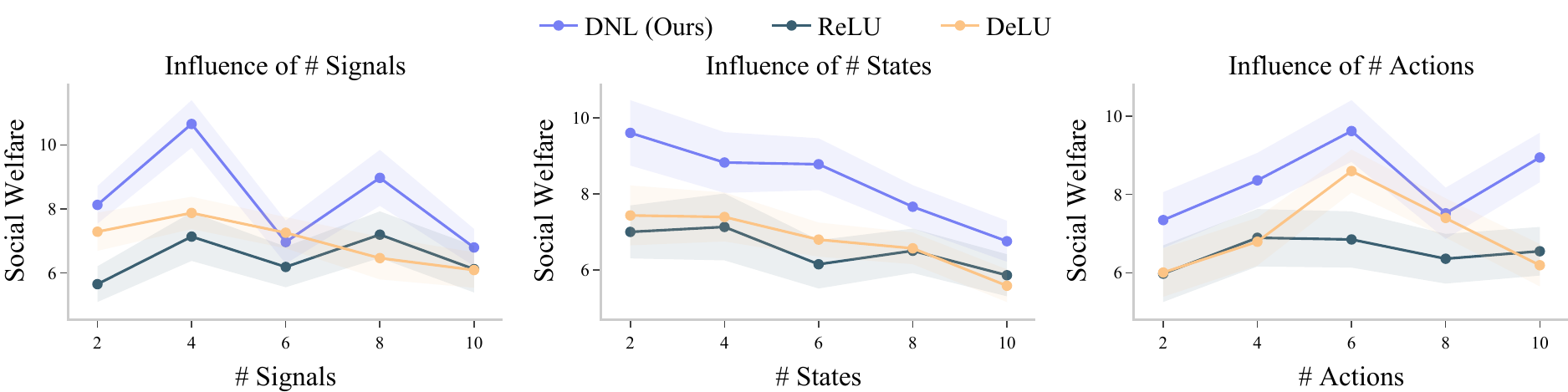}
    \includegraphics[width=\linewidth]{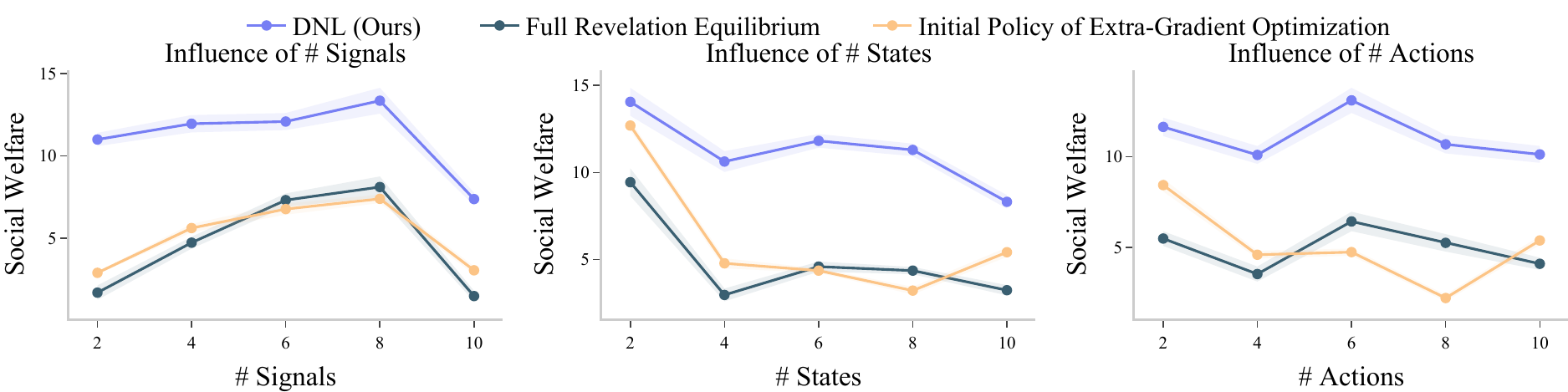}
    \vspace{-2em}
    \caption{Our method retains its advantage and achieves higher social welfare compared against baselines and full-revelation solutions when we adopt a stricter standard for the local NE check procedure (increase $\epsilon$ from 0.005 to 0.01).}
    \label{fig:2_sender_ln}
    \vspace{-1em}
\end{figure*}
\section{More Related Works}

Our work focuses on a type of Stackelberg game. Since the signaling strategy of principals could be continuous, this game has a continuous action space. Stackelberg games are employed in various real-world hierarchical scenarios, including taxation~\cite{zheng2020ai}, security~\cite{jiang2013defender, gan2020mechanism}, and business strategies~\cite{naghizadeh2014voluntary, zhang2016multi, aussel2020trilevel}. These games typically involve a leader and a follower. In such games with discrete choices, \citet{conitzer2006computing} demonstrate that linear programming can efficiently find Stackelberg equilibria using the strategy spaces of both players. For continuous decision spaces, \citet{jin2020local, fiez2020implicit} introduce and define local Stackelberg equilibria through first- and second-order conditions, with \citet{jin2020local} also showing that gradient descent-ascent methods can achieve these equilibria under certain conditions, and \citet{fiez2020implicit} providing specific updating rules that guarantee convergence.

With multiple followers~\cite{zhang2024social}, unless they operate independently \citep{calvete2007linear}, identifying Stackelberg equilibria is significantly harder and becomes NP-hard, even if followers have structured equilibria \citep{basilico2017methods}. \citet{wang2021coordinating} suggest managing an arbitrary equilibrium that the follower may reach through differentiation. Meanwhile, \citet{gerstgrasser2023oracles} develop a meta-learning framework across various follower policies, facilitating quicker adaptations for the principal. This builds on \citet{brero2022learning}, who pioneered the Stackelberg-POMDP model.

The field of multi-agent reinforcement learning~\cite{yu2022surprising,wen2022multi,qin2022multi,kuba2021trust,wang2019influence,christianos2020shared,peng2021facmac,jiang2019graph,wen2022multi,rashid2018qmix,wang2020roma, wang2021rode,wang2019learning,kang2020incorporating,li2021celebrating,wang2021off,guestrin2002coordinated, guestrin2002multiagent, bohmer2020deep,kang2022non,wang2021context,yang2022self,dong2022low,dong2023symmetry,li2023never,wu2021containerized} is expanding the application of Stackelberg concepts to more complex, realistic settings. \citet{tharakunnel2007leader} introduced a Leader-Follower Semi-Markov Decision Process for sequential learning in Stackelberg settings. \citet{cheng2017multi} developed a method known as Stackelberg Q-learning, albeit without proving convergence. Furthermore, \citet{shu2018m, shi2019learning} have empirically examined these leader-follower dynamics, focusing on the leader's use of deep learning models to predict follower actions.

\end{document}